\documentclass[a4paper]{amsart} 

\usepackage{amsmath,amsthm,amssymb,amsfonts,mathrsfs,color,hyperref, mathtools,crop, graphicx, enumitem, todonotes}
\usepackage[]{geometry}

\theoremstyle{plain}
\begingroup
\newtheorem{theorem}{Theorem}[section]
\newtheorem*{theorem*}{Theorem}
\newtheorem*{"theorem"}{``Theorem''}
\newtheorem{corollary}[theorem]{Corollary}

\newtheorem{lemma}[theorem]{Lemma}
\endgroup

\theoremstyle{definition}
\begingroup

\endgroup

\theoremstyle{remark}
\begingroup
\newtheorem{remark}[theorem]{Remark}
\newtheorem{example}[theorem]{Example}
\endgroup 

\numberwithin{equation}{section}
\newenvironment{pde}{\left\{\begin{array}{rll} } {\end{array}\right.}

\newcommand{\N}{\mathbb N}

\newcommand{\R}{\mathbb R} 
\newcommand{\E}{{\mathbb E}}
\renewcommand{\P}{{\mathbb P}}

\newcommand{\dist}{{\rm dist}}

\renewcommand{\div}{{\rm div}}

\newcommand{\spt}{{\mathrm{spt}}}

\newcommand{\dom}{\mathrm{dom}}

\renewcommand{\H}{{\mathcal H}}

\newcommand{\F}{{\mathcal F}}

\newcommand{\Ra} {\Rightarrow}
\newcommand{\wto}{\rightharpoonup}

\renewcommand{\d}{\mathrm{d}}
\newcommand{\dH}{\mathrm{d}\mathcal{H}}
\newcommand{\dx}{\,\mathrm{d}x}

\newcommand{\dz}{\,\mathrm{d}z}
\newcommand{\ds}{\,\mathrm{d}s}
\newcommand{\dt}{\,\mathrm{d}t}
\newcommand{\dr}{\,\mathrm{d}r}
\renewcommand{\P}{\mathbb{P}}

\newcommand{\eps}{\varepsilon}
\newcommand{\average}{{\mathchoice {\kern1ex\vcenter{\hrule height.4pt
width 6pt depth0pt} \kern-9.7pt} {\kern1ex\vcenter{\hrule
height.4pt width 4.3pt depth0pt} \kern-7pt} {} {} }}
\newcommand{\avint}{\average\int}

\allowdisplaybreaks

 \makeatletter
\@namedef{subjclassname@2020}{2020 Mathematics Subject Classification}
\makeatother
\newcommand\showlabel{\addtocounter{equation}{1}\tag{\theequation}}

\begin{document}

\title[SGD in machine learning: Implicit bias]{Stochastic gradient descent with noise of machine learning type\\{\small Part II: Continuous time analysis}}


\author{Stephan Wojtowytsch}
\address{Stephan Wojtowytsch\\
Department of Mathematics\\
Texas A\&M University\\
155 Ireland Street\\
College Station, TX 77840
}
\email{stephan@tamu.edu}

\date{\today}


\begin{abstract}
The representation of functions by artificial neural networks depends on a large number of parameters in a non-linear fashion. Suitable parameters of these are found by minimizing a `loss functional', typically by stochastic gradient descent (SGD) or an advanced SGD-based algorithm.

In a continuous time model for SGD with noise that follows the `machine learning scaling', we show that in a certain noise regime, the optimization algorithm prefers `flat' minima of the objective function in a sense which is different from the flat minimum selection of continuous time SGD with homogeneous noise.
\end{abstract}

\maketitle

\vspace{-12mm}


\section{Introduction}

A prototypical task in machine learning involves the approximation of a function in high dimension with respect to an unknown data distribution $\mu$. In practical terms, this may correspond to minimizing a functional of the form
\[\showlabel \label{eq mse loss}
L(\theta) = \frac1{2n}\sum_{i=1}^n \big|h(\theta, x_i) - y_i\big|^2
\]
where $\mathcal H = \{h(\theta, \cdot): \R^d\to \R \,|\,\theta\in\R^m\}$ is a parametrized function class and $(x_1,y_1), \dots, (x_n, y_n)$ are known data samples. While it is expensive to compute the gradient of $L$ if the number of data points $n$ is very large, it is often cheap to estimate the gradient by randomly selecting $1\leq i\leq n$ and estimating
\[\showlabel\label{eq batch gradient estimator}
\nabla L(\theta) \approx \big(h(\theta, x_i) - y_i\big)\,\nabla_\theta h(\theta,x_i).
\]
For parameter optimization, it is common to initialize $\theta_0$ from a probability distribution and optimize the parameters inductively according to a scheme
\[
\theta_{t+1} = \theta_t - \eta_t \,\frac1k \sum_{i \in I_t} \big(h(\theta, x_i) - y_i\big)\,\nabla_\theta h(\theta,x_i)
\]
where $I_t\subseteq\{1,\dots, n\}$ is a sample set of $k$ data pointsused in the $t$-th time step, and $\eta_t>0$ is a time-step size (or `learning rate'). In the limit of using all $k=n$ data points and infinitesimal step size, these training dynamics are captured by the gradient-flow 
\[
\dot \theta_t = - \nabla L(\theta_t).
\]
In this article, we consider the impact of a first order stochastic correction on the classical gradient flow, which is modelled by a stochastic differential equation (SDE)
\begin{equation}\label{eq intro sde}
\dot\theta_t = - \nabla L(\theta_t) + \sqrt{\eta\,\Sigma(\theta_t)}\,\d B_t.
\end{equation}
While non-stochastic gradient descent may easily get stuck in a local minimizer in a complex energy landscape, the stochasticity of SGD is believed to be beneficial to escape high loss local minima and saddle points. It is furthermore believed to select `flat' minima among the many minimizers of an objective function if the number $m$ of parameters exceeds the number $n$ of data points. These flat minima are believed to have beneficial generalization properties (i.e. perform well on unseen data). Thus in the long time limit, we expect $\theta_t$ to spend most of its time in the vicinity of minimizers of $f$ where the function is `flat' in some sense.

This intuition is motivated in part by statistical mechanics. If $L$ is the free energy functional of a collection of particles $\theta = (x_1,\dots, x_n)\in \R^{3n}$ and $\eta$ is a small, positive temperature, then \eqref{eq intro sde} governs the evolution of $\theta$ under small thermal vibrations. In the stationary setting, the particles are distributed in space according to the Boltzmann-distribution (or Gibbs distribution) with density
\[
\rho(\theta) = \frac{\exp\left(-\frac{L(\theta)}\eta\right)}{\int_{\R^{3n}}\exp\left(-\frac{L(\theta')}{\eta}\right)\d\theta'},
\] 
which assigns highest probability to the points where $L$ is low. The stationary setting is generally considered to describe the state of a system in the long term limit after a perturbation.

While the positive temperature assists a particle system in returning to equilibrium after a perturbation and overcome energy barriers along its way, it also prevents the system from achieving minimal energy, as oscillations persist in the long term. In machine learning, this has led to strategies of `lowering the temperature' by taking learning rates $\eta_t$ which approach $0$ as $t\to\infty$. Thus at finite time, we enjoy the benefits of stochastic noise, but we expect to reach a minimizer in the long term.

In the companion article \cite{discrete_sgd}, we demonstrate that it may not be necessary to decrease the learning rate to zero in machine learning applications. This stems from the fact that the homogeneous and isotropic noise which we know from mechanics may not resemble the noise we encounter in SGD in overparametrized learning applications. Namely, if $L(\theta) = 0$, then the gradient estimator \eqref{eq batch gradient estimator} estimates the gradient $\nabla L(\theta) =0$ correctly independently of the `batch' $I_t$, since $h(\theta, x_i) = y_i$ for all $1\leq i\leq n$.. Thus, at the global minimum, the stochastic noise vanishes. 

Objective functions $f$ and stochastic noise in overparametrized deep learning models have very special attributes, which we encode in the following simplified picture:
\begin{enumerate}
\item $f:\R^m\to \R$ is $C^1$-smooth and $f\geq 0$,
\item there exists a non-compact closed manifold $N$ of dimension $n< m$ such that $L(\theta) = 0$ if and only if $\theta \in N$,
\item the covariance-matrix $\Sigma$ of the stochastic noise satisfies $\|\Sigma(\theta)\| \leq C\,f(\theta)$, and
\item the rank of $\Sigma(\theta)$ is less than $m-n-1$ for all $\theta\in \R^m$.
\end{enumerate}
For details and examples of learning problems where these conditions are inappropriate, see \cite[Section 2]{discrete_sgd}. Stochastic gradient descent with machine learning type noise has interesting properties which are not captured by models with homogeneous isotropic noise. In this article, we only consider the influence of the fact that the magnitude of noise scales with the objective function and leave the impact of low rank for future study.

In particular, we show that the Boltzmann-distribution is a poor model for the long term dynamics of SGD with noise of ML type.

Let us make this more precise. If the learning rate $\eta$ is very small, but $\sqrt{\eta}$ is only moderately small, SGD is well-approximated by an SDE of the form \eqref{eq intro sde}. Using the link between SDEs and parabolic PDEs, we investigate the long term behavior of solutions to continuous time SGD with ML noise if the noise intensity scales exactly like the value of the objective function and the noise is isotropic. The objective functions we consider in this article are inspired by overparametrized problems in deep learning, so the set of global minimizers is a submanifold of the parameter space with high dimension and co-dimension. In a toy model, we observe two regimes:
\begin{enumerate}
\item $\sqrt{\eta}$ is large in terms of the co-dimension of the minimizer manifold and the scaling parameter of the stochastic noise. Then an invariant distribution with Lebesgue density $\rho =c\, f^\alpha$ of the continuous time SGD process exists with $\alpha<-1$ depending on $\eta$. For $\rho$ to be integrable, we require $f$ to grow rapidly at infinity in high dimension.
\item $\sqrt{\eta}$ is small in terms of the co-dimension of the minimizer manifold. The density $\rho = f^\alpha$ is an invariant measure in this case for suitable $\alpha$ depending on $\eta$, but fails to be integrable. Any invariant distribution is supported on the minimizer manifold.
\end{enumerate}

By analogy to the discrete time analysis in our companion paper \cite{discrete_sgd}, we expect solutions to SGD to reach a neighbourhood of the minimizer manifold in finite time and converge to the minimum linearly from that point on if the noise parameter is sufficiently small. 
Since solutions to SGD get trapped in sublevel sets of the objective function with high probability, we expect that the long time behavior of SGD depends on the initial condition when $\sqrt\eta$ is very small, and that SGD converges to a nearby minimum if it is initialized in a certain potential well. If the noise is larger, it is reasonable to assume that the law of solutions would converge to a unique invariant distribution independently of the initial condition. We note the technical caveat that for the degenerate parabolic equations involved, no such convergence has been established rigorously to the best of our knowledge.

In the toy model, we identify a threshold value $\bar\eta$ above which there exist invariant distributions which have a density with respect to Lebesgue measure. As we decrease $\eta$ towards $\bar\eta$, the invariant distributions collapse to the minimizer manifold $N$ and approach a probability distribution $\pi^*$ on $N$ which is absolutely continuous with respect to the area measure on $N$. The density of $\pi^*$ only depends on the eigenvalues of the Hessian of the objective function. The density is large if the eigenvalues of $D^2f$ in direction orthogonal to $N$ are small and vice versa, suggesting that solutions of SGD prefer flat minima if the learning rate is at the threshold value or just above.

Flat minimum selection can also be observed in classical continuous time SGD with homogeneous and isotropic Gaussian noise when the learning rate proxy/noise parameter $\eta$ approaches zero. In both cases, the notion of `flatness' of $f$ depends only on the eigenvalues of $D^2f$ in the directions orthogonal to $N$, but the precise notions of flatness differ. In the case that $N$ has co-dimension $2$, we show that the densities for classical homogeneous noise and ML noise satisfy
\[
\rho^{hom}(\theta) = \frac{c}{\sqrt{\lambda_1(\theta)\cdot \lambda_2(\theta)}}, \qquad \rho^{ML}(\theta) = \frac{c}{\mathrm{agm}(\lambda_1(\theta),\lambda_2(\theta))}
\]
where $\lambda_1, \lambda_2$ are the eigenvalues of $D^2f$ for eigenvectors orthogonal to $N$ and agm is the {\em algebraic-geometric mean} function. The algebraic-geometric mean -- a type of interpolation between algebraic mean $(\lambda_1+\lambda_2)/2$ and geometric mean $\sqrt{\lambda_1\lambda_2}$ -- is much less sensitive to one parameter being small than the geometric mean. Thus intuitively, a point at which one eigenvalue is small would be classified `flat' by homogeneous noise SGD whereas both eigenvalues have to be small or one eigenvalue has to be extremely small for the notion of flatness imposed by ML noise SGD. The main results of our analysis are stated in Theorem \ref{theorem critical flatness} and Remark \ref{remark comparing notions of flatness}.

The article is structured as follows. In the remainder of the introduction, we discuss the context of our results in continuous time SGD in machine learning. The main results are given in Section \ref{section continuous}. Open problems are discussed in \ref{section conclusion}. The proofs are postponed to the appendices.

\subsection{Context} 

Stochastic gradient descent algorithms have been an active field of study since the middle of the 20th century \cite{robbins1951stochastic}, before rising to great prominence in the context of deep learning around the turn of the century -- see also \cite{bottou2018optimization} for an overview. A continuum description was derived rigorously in \cite{li2017stochastic} as the SDE \eqref{eq intro sde} driven by Gaussian noise. Earlier works such as \cite{mandt2016variational,mandt2015continuous} used an SDE model heuristically under homogeneous noise assumptions, but with an eye on advanced optimizers such as Adagrad. 

Continuous time SGD in this context was studied e.g.\ in \cite{sirignano2020stochastic,sirignano2017stochastic} for possibly non-homogeneous but bounded noise and in a scaling that is the continuous time analogue of decaying learning rates. 
Different works have studied continuous time SGD under the assumption that the covariance matrix of noise $\Sigma$ is 
\begin{itemize}
\item isotropic and homogeneous \cite{welling2011bayesian,neelakantan2015adding,ge2015escaping,raginsky2017non} or
\item equal to the Hessian of the loss function \cite{jastrzkebski2017three, luo2020many, zhu2018anisotropic,smith2017bayesian,liu2020stochastic,hoffer2017train,xie2020diffusion}.
\end{itemize}
Neither heuristic captures realistic noise in machine learning, which is of low rank and vanishes where the objective function is zero. Anisotropy plays a key role in understanding the success of deep learning in real applications \cite{zhu2018anisotropic}. The Hessian heuristic is justified close to the global minimum, where it important directions in the Hessian, but not the fact that the noise vanishes. Despite occasional claims to the contrary, the Hessian heuristic is not equally justified at general critical points. The gradient, Hessian, and noise covariance matrix of the mean squared loss function \eqref{eq mse loss} are given by
\begin{align*}
\nabla L(\theta) &= \frac1{n}\sum_{i=1}^n \big(h(\theta, x_i) - y_i\big)\,\nabla_\theta h(\theta,x_i)\\
D^2 L(\theta) &= \frac2{n}\sum_{i=1}^n \nabla_\theta h(\theta, x_i) \otimes \nabla_\theta h(\theta,x_i) + \frac1{n}\sum_{i=1}^n \big(h(\theta, x_i) - y_i\big)\,D^2_\theta h(\theta,x_i)\\
\Sigma(\theta) &= \frac1{n}\sum_{i=1}^n \big(\big(h(\theta, x_i) - y_i\big)\,\nabla_\theta h(\theta,x_i) - \nabla L(\theta)\big) \otimes \big(\big(h(\theta, x_i) - y_i\big)\,\nabla_\theta h(\theta,x_i) - \nabla L(\theta)\big)
\end{align*}
respectively. If $\nabla L(\theta) = 0$, the covariance matrix simplifies to
\[
\Sigma(\theta) = \frac1{n}\sum_{i=1}^n \big(h(\theta, x_i) - y_i\big)^2\,\nabla_\theta h(\theta,x_i) \otimes \nabla_\theta h(\theta,x_i),
\]
which is comparable to the first term in the product $L\cdot D^2L$, but ignores the second, which may be important in sufficiently non-linear function models. Noise of this type was also considered in \cite{mori2021logarithmic}, at least in one space dimension. For analytic characterization of stochastic noise in the mini-batch setting for linear function models, see also \cite{ziyin2021minibatch}.

The Hessian heuristic produces interesting results in convex optimization, but is ill-defined in machine learning applications, where under fairly general conditions the Hessian is {\em not} positive semi-definite in {\em any} neighbourhood of a global minimum \cite[Appendix B]{discrete_sgd}.

As any covariance matrix is positive semi-definite, the Hessian approximation cannot be used even at points closed to the minimum, and further geometric conditions have to be imposed. The reason behind this is that the set of minimizers 
\[
N = L^{-1}(0) = \big\{ \theta\in \R^m : h(\theta,x_i) = y_i\:\forall\ 1\leq i\leq n\big\}
\]
of mean squared error loss $L$ as in \eqref{eq mse loss} generically is a submanifold of the parameter space with high dimension $m-n$ and co-dimension $n$ due to Sard's theorem (see \cite[Theorem 2.6]{discrete_sgd}). As the set of minimizers of a convex function is convex, we note the following: {\em If $f:\R^m\to [0,\infty)$ vanishes on an $n$-dimensional manifold $N$ and there exists a ball $B= B_r(\theta)$ such that $f$ is convex on $B$, then $N\cap B = V\cap B$ where $V$ is an $n$-dimensional affine subspace of $\R^m$.}

If $h:\R^m\times\R^d\to \R$ is non-linear in $\theta$, there is no reason for $N$ to be flat along its minimum, and if $h$ is sufficiently non-linear, then $L$ cannot be convex close to its minimum \cite[Appendix B]{discrete_sgd}. Considerations from convex optimization may therefore not be as meaningful as commonly assumed in the overparametrized setting in machine learning (i.e.\ when the number of parameters exceeds the number of data points).

Which minimizer out of the large and high-dimensional set of minimizers a stochastic gradient descent algorithm can select as a limit point is an important problem, related to questions of implicit regularization in machine learning. It has been suggested that minimizers where the objective function is `flat' are attractive to gradient descent algorithms, and various explanations have been proposed \cite{hochreiter1997flat, keskar2016large, patel2017impact, barrett2020implicit, smith2021origin, feng2021inverse}. An advantage of our continuous time approach is that the analysis allows us to give precise meaning to the notion of `flatness'.

In this work, we employ a similar SGD model driven by isotropic, but non-homogeneous Gaussian noise, whose intensity depends on the value of the objective function. Our model incorporates the scaling which makes noise vanish at the set of global minimizers, but does not incorporate the low rank of the noise.

The SGD optimization of {\em infinitely wide} neural networks was considered in \cite{hu2019mean} for two-layer neural networks and in \cite{jabir2019mean} for infinitely deep residual neural networks. In both studies, global convergence to the minimizer of an entropy-regularized problem is proved, but for stochastic noise which is homogeneous and isotropic. While the stochastic noise is part of the energy dissipation mechanism, in the homogeneous case it can be absorbed as a regularizer into the energy functional via an optimal transport interpretation. The key feature in this analysis is that the PDE
\[
0=\rho_t - \Delta \rho = \rho_t - \div\big(\rho\,\nabla \big(\log\rho)\big)
\] 
can be derived both from Gaussian oscillations as the diffusion equation, and as the Wasserstein-gradient flow of the entropy functional \cite{jordan1998variational}. This link has also been exploited in the `mean field theory' of neural networks to obtain global convergence results for non-stochastic gradient descent \cite{chizat2018global,Chizat:2020aa, relutraining}. The same analysis cannot easily be extended to stochastic noise of machine learning type.

It has more recently been suggested that mini-batch noise is heavy-tailed and should be described by a L\'evy process modelled on an $\alpha$-stable distribution for $\alpha\neq 2$ \cite{simsekli2019tail}.
The impact of heavy-tailed noise has also been used to analyze the performance of SGD and adaptive gradient algorithms with stochastic gradient estimates \cite{zhou2020towards}. Specifically, the authors demonstrate that ADAM, which keeps an exponentially decaying memory of past gradients, has `lighter tailed' noise than SGD. A clear effect on the type of limit that the optimizer selects can be observed.

General L\'evy processes are have cadlag paths (continuous from the right, limit from the left), but jump positive distances. This leads to faster diffusion described by a fractional order (non-local) partial differential equation rather than a classical parabolic equation \cite{west1997fractional}. Much recent progress has been made concerning diffusion equations driven by very rough (but full rank) noise \cite{ros2016regularity,fernandez2017regularity}.

The picture for the PDE description of low rank noise is less clear. We note that a theory for diffusion on manifolds tangent to a family of low-dimensional subspaces of the tangent space has been developed in the context of sub-Riemannian geometry \cite{hassannezhad2014sub,barilari2017heat,agrachev2019comprehensive}, albeit under rather strong conditions on the oscillation of the subspaces from point to point. Whether these geometric conditions are realistic in machine learning remains to be seen.

It should also be noted that in the derivation of continuum time models of SGD, only the quotient of learning rate and batch size appears in the limiting model. This does not capture all discrete time phenomena \cite{wu2018sgd}, where the stability of a minimizer may depend on the same parameters in a more complicated way. In the underparametrized setting and in discrete time, the geometry of the invariant distribution of a stochastic gradient algorithm has been linked to estimating the generalization error \cite{camuto2021fractal}. On the other hand, we remark that stochastic gradient descent with small positive learning rate converges to a minimizer in the overparametrized setting which we are investigating. In some cases, stochastic gradient descent can be proved to choose minimizers with better generalization than (classical) gradient descent \cite{pesme2021implicit}.

We note that an entirely different approach to continuous time stochastic gradient descent is presented in \cite{latz2020analysis}. 

\subsection{Notation and conventions}

All random variables are defined on a probability space $(\Omega,\mathcal A, \P)$ which remains abstract and is characterized mostly as expressive enough to support a random initial condition and stochastic differential equation. The dyadic product of two vectors $a, b$ is denoted by
\[
a\otimes b = a\cdot b^T, \qquad\text{i.e.} (a\otimes b)_{ij} = a_i b_j.
\]
We employ Einstein's sum convention (i.e.\ repeated indices in the same expression are summed over).

\section{Stochastic gradient descent with ML noise in continuous time}\label{section continuous}

The assumptions we make in this section are different from those in our discrete time analysis \cite{discrete_sgd}. In particular, for technical reasons we assume that the objective function has a {\em compact} set of minimizers (although this assumption can be weakened, see Example \ref{example non-compact minimizer set}) and grows faster than quadratically at infinity. The former is unrealistic in deep learning, the latter is incompatible with the assumption that the gradient of the objective function is globally Lipschitz-continuous, which is useful in the global analysis of discrete time SGD. The \L{}ojasiewicz condition we make in \cite{discrete_sgd} is replaced by a maximal non-degeneracy condition on the Hessian of the objective function at the minimum.

The noise in our toy model is isotropic (but not homogeneous). In particular, the noise has full rank. We therefore at most capture the effect of noise scaling on continuous time SGD in machine learning compared to classical SGD, but not the effect that singular noise has. Already this modification showcases the need for models which more precisely capture the nature of stochastic noise in machine learning applications.

\subsection{An SDE model for stochastic gradient descent}

As $\eta\to 0$, solutions of the SGD time-stepping scheme converge to solutions of the gradient-flow equation $\dot \theta_t = -\nabla f(\theta_t)$ with probability $1$ on any compact time-interval $[0,T]$. We can consider gradient-flow dynamics as a zeroth order continuum model for SGD. Here we heuristically derive a first order continuum model which retains information about the stochasticity in the algorithm. Recall that 
\begin{align*}
\theta_{t+1} &= \theta_t - \eta g(\theta_t,\xi_t) \showlabel \label{eq sgd}\\
	&= \theta_t - \eta \nabla f(\theta_t) + \sqrt{\eta}\,\big[\sqrt{\eta} \big(\nabla f(\theta_t) - g(\theta_t,\xi_t)\big)\big]\\
	&= \theta_t - \eta \nabla f(\theta_t) + \sqrt{\eta} y_t
\end{align*}
where $y_t = \sqrt{\eta} \big(\nabla f(\theta_t) - g(\theta_t,\xi_t)\big)$ is a random variable with mean $0$ and covariance matrix $ \Sigma' = \eta\Sigma$, if $\Sigma$ is the covariance matrix of the gradient estimators $g$. To leading order, this scheme therefore resembles 
\[
\theta_{t+1}' = \theta_t' - \eta\,\nabla f(\theta_t') + \sqrt{\eta}\, \sqrt{\Sigma'(\theta_t')}\,y_t'
\]
where $y_t'$ is a standard Gaussian random variable. If $\Sigma'$ did not depend on $\eta$, this would be a time-discretization of the stochastic differential equation
\[
\d \theta_t' = - \nabla f(\theta_t')\,\dt + \sqrt{\Sigma'(\theta_t')} \,\d B_t.
\]
As usual, the precise nature of the random noise does not matter for small time-steps, and it can be shown rigorously that 
\[\showlabel\label{eq sgd continuum}
\d \theta_t = -\nabla f(\theta_t)\,\dt + \sqrt{\eta\,\Sigma(\theta_t)} \,\d B_t
\]
is a first order continuous time model for the time-stepping scheme \eqref{eq sgd} in a precise sense, see \cite{li2015dynamics}. This approximation is valid in the situation that $\eta$ is small and $\eta \Sigma$ is of order $1$. If $\|\eta\Sigma\|\gg 1$, the continuous time approximation becomes invalid, while for $\|\eta\Sigma\|\ll 1$, the stochasticity can be ignored on finite time scales. 

In this note, we consider learning rates which remain strictly positive, i.e.\ $\eta$ does not depend on $t$. This is realistic in deep learning applications, where we keep the largest learning rate we can get away with. The approach may be theoretically justified due to the special scaling of the noise and may be superior to decaying learning rate schedules for which $\eta_t\to 0$, see also \cite{discrete_sgd}.

\subsection{On stochastic analysis and second order parabolic equations}

It is well-known that stochastic differential equations driven by Brownian motion are linked to second order partial differential equations of elliptic and parabolic type. We recall one particular result.

\begin{lemma}
Let $\theta_t$ be a solution of the SDE
\begin{equation}\label{eq generic sde}
\d \theta_t = \mu(t,\theta_t)\,\d t + \sqrt{\Sigma(t,\theta_t)}\,\d B_t.
\end{equation}
Then the law $\rho_t$ of $\theta_t$ is a very weak solution of the PDE
\begin{equation}\label{eq generic pde}
\partial_t\rho = \partial_i\partial_j \big(\rho \,\Sigma_{ij}\big) - \partial_i \big(\rho \mu_i\big)
\end{equation}
where summation over repeated indices is implied.
\end{lemma}

By very weak solution we mean that for all $t>0$ we have $\rho_t\in L^1(\R^d)$ and
\[
\frac{d}{dt} \int \psi(\theta)\,\rho(t,\theta)\,\d\theta = \int \left(\Sigma_{ij}\partial_i\partial_j\psi + \mu_i\,\partial_i\psi\right)\rho\,\d\theta
\]
for all $\psi\in C_c^2(\R^m)$. More precisely, this solution is strong in time and very weak in space.

\begin{proof}
Recall the multi-dimensional It\^o formula: If $f:\R^m\to\R$ is a $C^2$-function and the $\R^m$-valued random variable $\theta_t$ satisfies the SDE $\d \theta_t = \mu_t \,\d t + \sqrt{\Sigma}\,\d B_t$, then $Y_t:= f(\theta_t)$ satisfies
\[
\d Y_t= \big(\partial_t f + \nabla f \cdot \mu + D^2 f : \Sigma\big)\d t  + \nabla f \cdot \Sigma \cdot \d B_t
\]
where $A:B = a_{ij}b_{ij}$ \cite[Theorem 3.6 in Chapter 3]{karatzas2014brownian}. Here and in the following, we use the Einstein sum convention where summation over repeated indices is implies.
We conclude that for functions $\psi$ that do not depend on time we have
\begin{align*}
\frac{d}{dt} \int_{\R^d} \psi(\theta)\,\rho_t(\theta) \,\d\theta &= \frac{d}{dt}\E_{\theta\sim\rho_t} \big[\psi(\theta)\big]\\
	&= \frac{d}{dt} \,\E_{\theta_0\sim\rho_0}\big[\psi(\theta_t)\big]\\
	&= \E_{\theta_0\sim\rho_0}\big[\nabla \psi(\theta_t) \cdot \mu(\theta_t) + D^2 \psi(\theta_t) : \Sigma(\theta_t)\big]\\
	&= \E_{\theta\sim\rho_t} \big[\nabla \psi\cdot \mu + D^2 \psi: \Sigma\big]
\end{align*}
since the diffusion term is a martingale and independent of $\mathcal \F_{t-}$, so it does not influence the evolution of the expectation.
\end{proof}

\subsection{Invariant measures}

If $\theta_0$ is initialized according to a distribution with density $\rho_0$ and $\theta_t$ evolves by the SDE \eqref{eq generic sde}, then the solution $\rho_t$ of the PDE \eqref{eq generic pde} with initial condition $\rho_0$ describes the law of $\theta_t$. Under fairly generic smoothness conditions, the existence of a unique solution $\rho_t$ such that $\rho_t\in L^1(\R^m)$ for all $t>0$ can be established. Regularity theory can be used to show that $\rho$ is smooth if $\mu:\R^m\to\R^m$ and $\Sigma:\R^m\to\R^{m\times m}$ are. 

A distribution $\pi_0$ is called {\em invariant} if the law of $\theta_t$ is constant in time for $\theta_0\sim\pi_0$. If $\pi_0$ has a density $\rho_0$ with respect to Lebesgue measure, then 
\begin{equation}\label{eq invariant measure}
0 = \partial_i\partial_j \big(\rho \,\Sigma_{ij}\big) - \partial_i \big(\rho \mu_i\big)
\end{equation}
in the very weak sense. If $\rho \in L^1_{loc}(\R^m)$ is a non-negative solution of \eqref{eq invariant measure} and $\rho\geq 0$, then $\rho$ is called an {\em invariant measure}. This is defined even if $\rho \notin L^1(\R^m)$, i.e.\ if $\rho$ cannot be normalized to a probability distribution. The equation \eqref{eq invariant measure} can be interpreted in the very weak sense even if $\rho$ is merely a measure.

In the case of continuous time SGD, $\mu = -\nabla f$, so $\rho$ satisfies
\begin{equation}\label{eq distribution gradient flow}
\rho_t = \div\big(\rho \nabla f\big) + \partial_i\partial_j \big(\rho\,\Sigma_{ij}\big).
\end{equation}
The type of noise is encoded in the structure of the matrix $\Sigma$.

\begin{example}\label{example boltzmann distribution}
For SGD with homogeneous isotropic noise, we have $\Sigma = {\eta}I$, so the invariant distribution satisfies
\[
\div(\rho\,\nabla f) + {\eta}\,\Delta \rho =0.
\]
It is easy to verify that $\hat\rho(\theta) = \exp\left(-\frac{f(\theta)}{\eta}\right)$ satisfies
\begin{align*}
\div(\hat\rho\,\nabla f) &= \div\left(\exp\left(-\frac{f(\theta)}{\eta}\right)\,\nabla f\right)\\
	&= -{\eta}\,\div\left(\nabla\exp\left(-\frac{f(\theta)}{\eta}\right)\right)\\
	&= -\eta \,\div\big(\nabla\hat\rho\big)\\
	&= - \eta\,\Delta\hat\rho,
\end{align*}
so
\[
\div(\hat\rho\,\nabla f) + {\eta}\,\Delta \hat\rho =0.
\]
If $f(\theta) \geq |\theta|^\kappa-C$ for some $C\in\R$ and $\kappa>0$, then $\hat\rho$ is integrable, so there exists $c>0$ such that $c\,\hat \rho$ is an invariant distribution for homogeneous isotropic noise SGD. If $f$ grows only like $f_0\log (1+|\theta|)$ and $\eta$ is large, then $\exp(-f/\eta)\sim (1+|\theta|)^{-f_0/\eta}$ fails to be integrable. An invariant distribution does not exist because particles may escape to infinity.
\end{example}

\subsection{Invariant distribution with ML-type noise}

As we observed in \cite{discrete_sgd}, bounds of the form
\[
\|\Sigma(\theta)\| \leq C\,f(\theta)\qquad\text{or}\qquad \|\Sigma(\theta) \| \leq C\,f(\theta)\,\big[1+|\theta|^2\big]
\]
are expected to hold in machine learning applications. We consider the very simplest noise model in this class: $\Sigma = \eta\sigma f I_{m\times m}$. Assume that $\rho_\infty$ is an invariant measure for the SDE \eqref{eq sgd continuum}. Then, using Einstein sum convention, we have
\begin{align*}
0 &= \partial_i\partial_j\big(\rho_\infty\,\Sigma_{ij}\big) + \partial_i\big(\rho_\infty\,\partial_i f\big)\\
	&= \partial_i \left( \partial_j \big(\rho_\infty\,\eta\,\sigma f\,\delta_{ij}\big) + \rho_\infty\,\partial_if\right)\\
	&= \partial_i\left(\eta\sigma f\,\partial_i\rho_\infty + (1+\eta\sigma) \rho_\infty\,\partial_if\right)\\
	&= \div \left(\eta \sigma \,f\nabla \rho_\infty + (1+\eta\sigma)\,\rho_\infty\nabla f\right)\\
	&= \eta\sigma\, \div\left( f^{1-\frac{1+\eta\sigma}{\eta\sigma}}\nabla \left(\rho_\infty \,f^\frac{1+\eta\sigma}{\eta\sigma}\right)\right)\\
	&= \eta\sigma\, \div\left( f^{-\frac{1}{\eta\sigma}}\nabla \left(\rho_\infty \,f^\frac{1+\eta\sigma}{\eta\sigma}\right)\right)\showlabel \label{eq invariant measure special}
\end{align*}
where all identities hold in the distributional sense. The product rule can be applied away from the set $N = \{f=0\}$ since $f$ and $\Sigma$ are smooth. 
In particular, for any $c>0$ we see that $\rho_\infty = c\,f^\alpha$ is a stationary and non-negative solution of \eqref{eq distribution gradient flow} for
\[
\alpha = - \frac{1+\eta\sigma}{\eta\sigma} < -1.
\]
This solution is meaningful as an invariant measure if $\rho_\infty$ is locally integrable and as an invariant distribution if $\rho_\infty$ is integrable. We consider $f$ in the following theorem to be a toy model for the energy landscape of overparametrized regression problems.

\begin{lemma}\label{lemma integrability}
Assume that 
\begin{enumerate}
\item There exists a compact $n$-dimensional manifold $N\subseteq\R^m$ such that $f(\theta) = 0$ if and only if $\theta\in N$,
\item $D^2f(\theta)$ has rank $m-n$ for all $\theta\in N$, and
\item there exist $\gamma, R, c_1>0$ and $f(x) \geq c_1\,|x|^\gamma$ for all $|x|\geq R$.
\end{enumerate}
Then $f^{\alpha}$ is locally integrable if $\alpha> -\frac{m-n}2$ and $f^\alpha$ is integrable if and only if additionally $\alpha < -\frac m\gamma$.
\end{lemma}

Note that here $n = \dim(N)$, whereas in \cite[Theorem 2.6]{discrete_sgd} we have $\dim(N) = m-nk$. This incompatibility in notation between different parts of this article is chosen to make the individual parts more natural in their notation. The proof of Lemma \ref{lemma integrability} is given in Appendix \ref{appendix flat minimum}.

\begin{remark}
Even when it is defined, the invariant distribution $\rho = c\,f^\alpha$ is non-unique, since any probability distribution which is supported on $N$ is invariant under continuous time SGD dynamics. Whether or not it is unique in the class of distributions which admit a density with respect to Lebesgue measure, or whether the law of solutions to continuous time SGD with generic intialization approaches it in the long time limit, remains an open question in many cases.

For partial results, see Theorems \ref{theorem convergence}, \ref{theorem unique invariant distribution} and \ref{theorem convergence appendix}.
\end{remark}

\begin{remark}
We note that $f^\alpha$ cannot be globally integrable if $f$ grows quadratically, since this would require $ - \frac{m-n}2 < \alpha < -\frac m2$. Thus the Lemma cannot be used in the case that $f$ has Lipschitz-continuous gradients.

The two conditions play very different roles:
\begin{enumerate}
\item If $\alpha = -\frac{1+\sigma\eta}{\sigma\eta}$ is large negative, the random noise (which scales like $\sqrt{\eta\sigma}$) is very small close to $\{\theta : f(\theta) =0\} = \{\theta : \Sigma(\theta) = 0\}$. Close to $N$, the function $f$ satisfies a \L{}ojasiewicz inequality since $D^2f$ has maximal rank on $N$. If a particle enters such a neighbourhood, we expect it to converge to $N$ linearly (by resemblance to GD or by appealing to Theorem \cite[Theorem 3.3]{discrete_sgd}). 

Thus the invariant distribution cannot have a positive density close to $N$ as particles collapse onto the manifold. If the noise is positive but very small, the invariant distribution concentrates on $N$.

If the co-dimension $m-n$ of $N$ is large, it is more likely that a random perturbation causes a particle to `miss' $N$. Thus if $N$ has high co-dimension, the invariant distribution has a density $c\,f^{-\frac{1+\eta\sigma}{\eta\sigma}}$ for lower noise levels $\eta\sigma$ than if the co-dimension of $N$ is small.  

\item If $\alpha$ is close to $-1$, this corresponds to a large noise parameter $\eta\sigma$. Since $ \Sigma = \eta\sigma f\,I$, also the noise at infinity becomes large, and particles may escape towards infinity despite the presence of the gradient term. For simplicity, consider $f(x) = |x|^\gamma$. Then the magnitude of noise scales like $\sqrt{f(x)} = |x|^{\gamma/2}$, which is easily dominated by the gradient term $|\nabla f(x)| \sim |x|^{\gamma-1}$ if $\gamma$ is large enough. The growth condition induces a bound on average gradients in a suitable sense which we may interpret as the statement that SGD with ML type noise traps particles almost surely in a sublevel set.
\end{enumerate}
\end{remark}

The assumption that $N$ is compact simplifies the result, but is not technically necessary (and, considering Theorem \cite[Theorem 2.6]{discrete_sgd}, not realistic in overparametrized machine learning). Non-compactness of the manifold $N$ can be compensated if the minimum becomes sufficiently steep at infinity.

\begin{example}\label{example non-compact minimizer set}
Let $-1.5 < \alpha< -1$ and
\[
f(\theta_1, \theta_2, \theta_3, \theta_4) = \big(\theta_1^2 +\theta_2^2+\theta_3^2\big) \big(1+ \theta_1^2+\theta_2^2 + \theta_3^2\big)\big(1+ \theta_4^2\big).
\]
Then $N = \{\theta \in \R^4 : \theta_1 = \theta_2 = \theta_3 = 0\}$ and
\[
\int_{\R^4} f^\alpha(\theta) \,\d\theta = \left(\int_{-\infty}^\infty (1+\theta_4^2)^\alpha\,\d\theta_4\right)\left(\int_{\R^3} \big(\theta_1^2 +\theta_2^2+\theta_3^2\big)^\alpha \big(1+ \theta_1^2+\theta_2^2 + \theta_3^2\big)^\alpha \,\d\theta_1\,\d\theta_2\,\d\theta_3\right)<\infty
\]
by the same considerations as above.
\end{example}

However, we note that if the minimum becomes very steep at infinity, finite step size effects in SGD are likely to dominate the picture.

\begin{remark}
The isotropic model $\Sigma = \eta\sigma f\,I$ is unlikely to capture important aspects of machine learning since the type of noise in this model is the same on the entire level set $\{\theta : f(\theta) = \eps\}$. If
\[
f(\theta) = L(\theta) = \int_{\R^d\times \R} \big(h(\theta,x)-h(\theta^*, x)\big)^2\,\d\bar\mu_{x},
\]
then a level set can contain two very different events:
\begin{itemize}
\item $\theta$ is close to $\theta^*$ and the $L(\theta)=\eps$ with small fluctuations over the entire data set.
\item $\theta$ is far away from $\theta^*$, but 
\[
\P\big(\{x : h(\theta,x)= h(\theta^*,x)\}\big) = 1-\eps, \qquad \P\big(\{x : h(\theta,x)- h(\theta^*,x)=1\}\big) = \eps.
\]
This can occur when the distribution has a majority phase on which $h(\theta,\cdot)$ and $h(\theta^*, \cdot)$ behave similarly and a minority phase on which the functions are very different.
\end{itemize}
There is no reason to believe that the noise in both cases would be similar.
\end{remark}

\begin{remark}
In the underparametrized regime, data points cannot be fitted perfectly by the machine learning model. This corresponds to minimizing an objective function with ML type noise such that $\inf_\theta f(\theta)>0$. If this is the case, we only require the integrability condition $- \frac{1+\eta\sigma}{\eta\sigma} = \alpha < -\frac m\gamma$, so the growth condition on $f$ can be relaxed. For any $\gamma>0$, we can choose a noise level $\eta\sigma$ so small that the invariant distribution $\rho= c\,f^\alpha$ exists.
\end{remark}

\subsection{Convergence to the invariant distribution}

Whether or not the invariant distribution we identified reflects the long term dynamics of continuous time SGD with ML noise is a subtle question. At least in the underparametrized regime, we can show that no other invariant distribution exists. Under stronger assumptions, we show that the invariant distribution is achieved exponentially fast in the long term limit.

\begin{theorem}\label{theorem convergence}
Assume that there exists constants $0<\lambda \leq \Lambda$ such that
\[
\lambda\big(1+|\theta|^2\big) \leq f(\theta)\leq \Lambda \big(1+|\theta|^2\big)\qquad\forall\ \theta\in \R^m
\]
and that $\eta\sigma < \frac{2}{m-2}$ such that $\alpha = -1-\frac1{\eta\sigma} < -\frac m{2}$. Let $\rho_0$ be a probability density on $\R^m$ such that
\[
\int_{\R^m} \rho_0^2\,f^{1+\frac1{\eta\sigma}}\,\d\theta < \infty.
\]
Then there exists a solution $\rho$ of the equation 
\[
\partial_t\rho = \div\left( \eta\sigma f\,\nabla \rho + (1+\eta\sigma)\rho\,\nabla f\right)
\]
which describes the evolution of the density of a solution to SGD with noise model $\Sigma = \sigma f\,I$ (compare \eqref{eq invariant measure special}).
Furthermore, there exists $\nu>0$ depending only on $m, \eta\sigma, \lambda, \Lambda$ such that
\begin{align*}
\int_{\R^m}& \left( \rho(t,\theta) - \frac{f^{-1-\frac1{\eta\sigma}} (\theta)}{\int_{\R^m}f^{-1-\frac1{\eta\sigma}}(\theta')\,\d\theta'}\right)^2\,f^{1+\frac1{\eta\sigma}}(\theta)\,\d\theta \\
	&\hspace{2cm}\leq e^{-\nu t} \int_{\R^m} \left( \rho(0,\theta) - \frac{f^{-1-\frac1{\eta\sigma}} (\theta)}{\int_{\R^m}f^{-1-\frac1{\eta\sigma}}(\theta')\,\d\theta'}\right)^2\,f^{1+\frac1{\eta\sigma}}(\theta)\,\d\theta.
\end{align*}
In particular $\rho(t,\cdot)$ converges to the invariant distribution in $L^2(\R^m)$.
\end{theorem}

More plainly, the law of a solution to continuous SGD with isotropic noise of ML type converges to the invariant distribution if the objective function grows algebraically and sufficiently fast at infinity and is bounded away from zero. The proof of Theorem \ref{theorem convergence} is given in Appendix \ref{appendix convergence}, where $\nu$ is estimated as well. The proof may be extended to the case where $f$ vanishes at a finite number of points, but not along a positive-dimensional manifold. Furthermore, the objective function has to detach faster than quadratically from its global minimizers, meaning that it cannot be $C^2$-regular (or in fact $C^{1,1}$-regular) close to its minimizers. The precise result is given in Theorem \ref{theorem convergence appendix} in Appendix \ref{appendix poincare}.

The following result poses fewer conditions on the objective function $f$, but only establishes the uniqueness of the invariant distribution, not the convergence to it at a certain rate.

\begin{theorem}\label{theorem unique invariant distribution}
Assume that
\begin{enumerate}
\item $\inf_\theta f(\theta)>0$,
\item $f$ is an element of the H\"older space $C^{2,\alpha}_{loc}(\R^m)$ and
\item there exists $C>0$ such that $|\nabla \log f|(\theta) \leq \frac{C}{1+|\theta|}$.
\end{enumerate}
Then there exists a unique non-negative solution $\rho$ of the stationary distribution equation \eqref{eq invariant measure special} (up to multiplication by a positive constant). 
\end{theorem}

Note that the decay condition on the gradient is for example met if there exist $R, \gamma>0$ such that$f(\theta) = |\theta|^\gamma$ for $|\theta|\geq R$, but also allows for mixed growth of the form $f(\theta_1, \theta_2) = \theta_1^2 + \theta_4^2$. The proof is given in Appendix \ref{appendix invariant distribution}. Our proof cannot be extended to the underparametrized case -- see however Example \ref{example non-uniqueness pde} following the proof of Theorem \ref{theorem unique invariant distribution} for a discussion of potential non-uniqueness.

As a consequence of Theorem \ref{theorem unique invariant distribution}, if $f$ is bounded away from zero, but grows too slowly at infinity compared to the size of $\eta\sigma$, an invariant distribution of continuous time SGD with noise $\Sigma=\eta\sigma f\,I$ does not exist as the unique invariant measure cannot be normalized.

\subsection{Flat minimum selection}
We see that the invariant measure $\rho_{\alpha} := f^\alpha$ concentrates close to the set $N$ of global minimizers of $f$ for $\alpha> -\frac{m-n}2$. We think of this as {\em global minimum selection}: If SGD is run for long enough with constant step size, eventually it will spend most of its time close to $N$, independently of the initial condition. This is true also for SGD with small, but uniform Gaussian noise (see Example \ref{example boltzmann distribution} and Theorem \ref{theorem flat minima homogeneous} below).

For $\alpha < - \frac{m-n}2$, the invariant distribution of continuous time SGD with ML type noise does not have a density with respect to Lebesgue measure. In this section, we investigate the limiting behavior as $\alpha\to - \frac{m-n}2$, i.e.\ as the invariant distributions collapse onto $N$. In this analysis, we show that continuous time SGD with ML type noise mostly concentrates at those points on $N$ where the energy landscape is `flat' in a particular sense.

\begin{theorem}\label{theorem invariant distribution critical sgd noise}\label{theorem critical flatness}
Let $f\in C^2(\R^m,[0,\infty)) $, and assume that
\begin{enumerate}
\item There exists a compact $n$-dimensional manifold $N\subseteq\R^m$ such that $f(\theta) = 0$ if and only if $\theta\in N$,
\item $D^2f(\theta)$ has rank $m-n$ for all $\theta\in N$, and
\item there exist $\gamma >\frac{2m}{m-n}$, $R>0$ and $c_1>0$ such that $f(\theta) \geq c_1\,|\theta|^\gamma$ whenever $|\theta| \geq R$.
\end{enumerate}
Then for all $\alpha\in\left(- \frac{m-n}2, - \frac{m}\gamma\right)$, there exists an invariant distribution $\rho_\alpha := c(\alpha)\,f^\alpha$ such that $\int_{\R^m}\rho_\alpha(\theta)\,\d\theta = 1$. As $\alpha\to\alpha^* = -\frac{m-n}2$, the measures $\pi_\alpha$ which have density $\rho_\alpha$ with respect to Lebesgue measure converge to a probability measure $\pi^*$ in the sense of Radon measures such that
\[
\pi^*(\R^m\setminus N) = 0
\]
and $\pi^*$ has a density $\rho^*$ with respect to the $n$-dimensional Hausdorff measure on $N$ which is proportional to
\[
\hat\rho^*(\theta) = \int_{S^{n-m-1}} \big(\nu^T \,\widehat{D^2f(\theta)}\nu\big)^{-\frac{m-n}2}\,\d\H^{m-n-1}_\nu
\]
where 
\[
\widehat{D^2f(\theta)} = \mathrm{diag} (\lambda_1,\dots,\lambda_{m-n})
\]
is the $(m-n)\times (m-n)$ matrix with the non-zero eigenvalues of $D^2f$ on the diagonal (the diagonal matrix in a reduced singular value decomposition of $D^2f(\theta))$.
\end{theorem}

All proofs for this section are given in Appendix \ref{appendix flat minimum}. Our analysis of the invariant distributions suggests that SGD with ML type noise preferentially approaches `flat' minima of $f$. This is also the case for SGD with uniform Gaussian noise in the small noise limit, but for a different notion of flatness. The main difference is that we take the zero noise limit for uniform Gaussian noise, but a specific `small finite noise' limit for ML type noise.

\begin{theorem}\label{theorem flat minima homogeneous}\label{theorem integrability}
Let $f\in C^2(\R^m,[0,\infty)) $, and assume that
\begin{enumerate}
\item There exists a compact $n$-dimensional manifold $N\subseteq\R^m$ such that $f(\theta) = 0$ if and only if $\theta\in N$,
\item $D^2f(\theta)$ has rank $m-n$ for all $\theta\in N$, and
\item there exist $\gamma >0$, $R>0$ and $c_1>0$ such that $f(\theta) \geq c_1\,|\theta|^\gamma$ whenever $|\theta| \geq R$.
\end{enumerate}
Then for all $\eta>0$, there exists an invariant distribution $\rho_\eta := c_\eta\,\exp\big(-f/{\eta}\big)$ such that $\int_{\R^m}\rho_\eta(\theta)\,\d\theta = 1$. The measures $\pi_\eta$ which have density $\rho_\eta$ with respect to Lebesgue measure converge to a probability measure $\pi^*$ in the sense of Radon measures as $\eta\to0$. The limit $\pi^*$ satisfies $\pi^*(\R^m\setminus N) = 0$ and has a density $\rho^*$ with respect to the uniform distribution/$n$-dimensional Hausdorff measure on $N$ which is proportional to
\[
\hat\rho^*(\theta) = \det\big(\widehat{D^2f(\theta)}\big)^{-\frac{1}2}
\]
where 
\[
\widehat{D^2f(\theta)} = \mathrm{diag} (\lambda_1,\dots,\lambda_{m-n})
\]
is the $(m-n)\times (m-n)$ matrix with the non-zero eigenvalues of $D^2f$ on the diagonal (the diagonal matrix in a reduced singular value decomposition of $D^2f(\theta))$.
\end{theorem}

\begin{remark}\label{remark comparing notions of flatness}
Note that the functions $g_1, g_2: (0,\infty)^{m-n}\to (0,\infty)$
\begin{align*}
g_1(\lambda_1,\dots,\lambda_{m-n}) &= \big(\det \big(\mathrm{diag}(\lambda_1,\dots,\lambda_{m-n})\big)^{-1/2}\\
g_2(\lambda_1,\dots,\lambda_{m-n}) &= \avint_{S^{n-m-1}} \big(\nu^T \,\mathrm{diag}(\lambda_1,\dots,\lambda_{m-n})\nu\big)^{-\frac{m-n}2}\,\d\H^{m-n-1}_\nu
\end{align*}
which measure the steepness of a local minimum both satisfy 
\[
g(1,\dots,1) = 1, \qquad g(\mu\,\lambda) = \mu^{-\frac{m-n}2} g(\lambda).
\]
In that sense, the notions of flatness are comparable, but they are not the same if $m-n\geq 2$. In particular, if $m-n=2$, we note that $g_1(\eps, 1) = \eps^{-1/2}$ while
\[
g_2(\eps,1) = \frac1{2\pi} \int_0^{2\pi} \frac1{\sqrt{\cos^2(\phi) + \eps\,\sin^2(\phi)}}\,\d\phi = \mathrm{agm}(1,\eps)^{-1}
\]
is a complete elliptic integrals of first type, which evaluates to the inverse of the {\em algebraic geometric mean} of $1$ and $\eps$, which interpolates between the geometric mean $\sqrt{1\cdot\eps}$ of $1$ and $\eps$ (which is small) and the algebraic mean $\frac{1+\eps}2$ of $1$ and $\eps$ (which is large). Since
\[
\lim_{\eps \to 0} \left[|\log\eps| \cdot \mathrm{agm}(1,\eps)\right] = \frac\pi 2,
\]
we find that $g_2(\eps,1) \ll g_1(\eps,1)$ if $\eps$ is small. Intuitively, $g_1$ classifies $f$ as flat at $\theta\in N$ if the Hessian has a single very small eigenvalue, while $g_2$ requires {\em all} eigenvalues to be small, or one to be extremely small.

\begin{figure}
\begin{center}
\includegraphics[width = 0.4\textwidth]{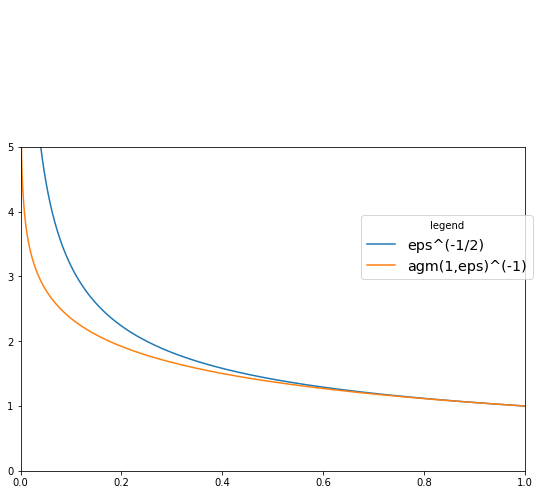}\hfill
\includegraphics[width = 0.4\textwidth]{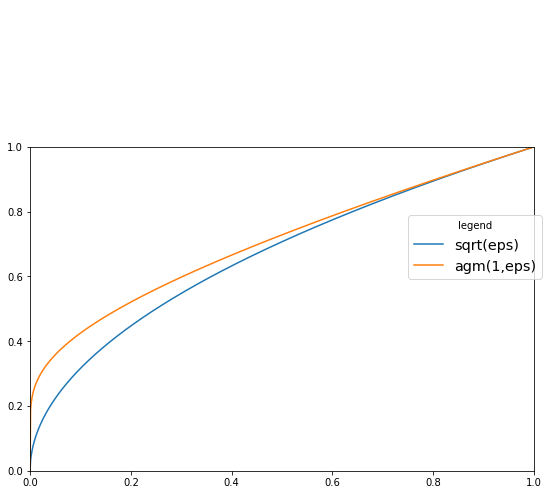}
\end{center}
\caption{\label{figure agm vs sqrt}The arithmetic geometric mean approaches zero much more slowly than the geometric mean, so its inverse diverges much more slowly at $0$.}
\end{figure}
\end{remark}

\begin{remark}
Theorem \ref{theorem invariant distribution critical sgd noise} states that if the effective noise parameter $\eta\sigma$ just exceeds the critical value where $-\alpha = \frac{1+\eta\sigma}{\eta\sigma} = \frac{m-n}2$, there exists an invariant distribution which concentrates close to $N$ and assigns highest probability to points at which $f$ is flat. We have reason to expect that the law of $\theta_t$ will approach this invariant distribution independently of how $\theta_0$ is initialized.

If the effective noise parameter $\eta \sigma$ is smaller than the critical threshold $\frac{m-n}2$, we expect the `terminal distribution' $\rho_\infty = \lim_{t\to\infty}\rho_t$ to depend on the initial condition $\theta_0$. For instance, if $N$ has two connected components which are separated by a significant potential barrier and $\theta_0$ is initialized close to one component, the probability of escape from either potential well is small. Once a particle reaches the domain where the objective function is small, the noise is also small, and the particle approaches a nearby point on $N$ with high probability. For small $\eta\sigma$, the noise is not strong enough to force $\theta_t$ to explore a significant portion of the parameter space, or even the entire set of minimizers $N$. Therefore, we expect the long time behaviour of $\theta_t$ to depend on how much probability the initial distribution assigned to the individual potential wells.

For any positive noise $\eta\sigma>0$, the invariant distribution of continuous time SGD \eqref{eq distribution gradient flow} is non-unique since {\em any} probability measure $\pi$ on $N$ is an invariant distribution.
\end{remark}

Note that the invariant distribution under isotropic homogeneous noise is invariant under shifts of the objective function $f\mapsto f+c$ since $\exp((f+c)/\eta) = \exp(c/\eta)\,\exp(f/\eta)$, so any objective shift is absorbed into the normalizing constant. The situation in the underparametrized regime with machine learning noise resembles the case of homogeneous noise more than that of ML noise in the overparametrized regime. This is not surprising, since either version of SGD is expected to spend most of its time close to the global minimum on a long enough time scale. Both noise models are isotropic, and the slight variation which $\Sigma = \sigma\,f I$ experiences in the set $\{\theta : f(\theta) < (1+\mu)\min f\}$ for small $\mu>0$ becomes negligible.

\begin{theorem}\label{theorem flatness ml underparametrized}
Let $f\in C^2(\R^m,(0,\infty)) $, and assume that
\begin{enumerate}
\item $\eps:= \inf_{\theta}f(\theta)>0$
\item There exists a compact $n$-dimensional manifold $N_\eps\subseteq\R^m$ such that $f(\theta) = \eps$ if and only if $\theta\in N_\eps$,
\item there exist $\gamma >0$, $R>0$ and $c_1>0$ such that $f(x) \geq c_1\,|\theta|^\gamma$ whenever $|\theta| \geq R$.
\end{enumerate}
Then for all $\eta>0$ if $\gamma>m$ and all $\eta \in \left(0, \,\frac{\gamma}{\sigma(m-\gamma)}\right)$ otherwise, there exists an invariant distribution $\rho_\eta := c_\eta\,f^{-\frac{1+\eta\sigma}{\eta\sigma}}$ such that $\int_{\R^m}\rho_\eta(x)\dx = 1$. The measures $\pi_\eta$ which have density $\rho_\eta$ with respect to Lebesgue measure converge to a probability measure $\pi^*$ supported on $N_\eps$ as $\eta\searrow 0$. The limiting distribution $\pi^*$ has a density which is proportional to
\[
\tilde \rho^*(\theta) = \det\big(\widehat{D^2f(\theta)}\big)^{-\frac{1}2}
\]
with respect to the uniform distribution on $N_\eps$.
\end{theorem}

Remarkably, in all three cases, the limiting invariant distribution is indepent of the local geometry of $N$ (curvature, reach, ...) and only depends on the Hessian of $f$. We give Theorem \ref{theorem flatness ml underparametrized} for the sake of completeness, but note that in underparametrized learning, the set of minimizers is generally zero-dimensional.

\section{Conclusion}\label{section conclusion}

The precise nature of stochastic noise in SGD is important for the global convergence, long time asymptotics and minimum selection. Both homogeneous isotropic and Hessian noise are inadequate in machine learning applications. In a toy model for machine learning SGD which exhibits the right scaling behavior at the minimum, we establish the following:

\begin{enumerate}
\item Stochastic gradient descent may converge to a global minimizer even with uniformly positive learning rate, if the stochastic noise is of machine learning type (see also \cite{discrete_sgd}).
\item If the learning rate is slightly larger than a critical threshold, an invariant distribution of continuous time SGD exists, which concentrates around `flat' minimizers of $f$ in a precise sense.
\item The flatness condition is distinct from the one that homogeneous noise induces.
\end{enumerate}

There are many open problems. Firstly, while our noise scales appropriately with the objective function, it is isotropic and thus of full rank, while noise in overparametrized learning typically has low rank. The driving noise of our SDE is Gaussian, leading to a second order PDE description of the law of solutions. Heavy-tailed distributions would lead to more complicated non-local, fractional order PDEs (of elliptic type for the invariant measure and parabolic type in general).

Furthermore, the SDE model describes SGD well for finite times, but not necessarily in the long time limit. The asymptotic analysis of invariant distributions therefore is only indicative of particular behaviors on a heuristic level. Additionally, we only prove that the law of solutions to continuous time SGD converges to the invariant distribution under very restrictive conditions.

\bibliographystyle{../../alphaabbr}
\bibliography{../../NN_bibliography}

\newpage
\appendix

\section{Proof of Theorem \ref{theorem unique invariant distribution}: Uniqueness of the invariant measure}\label{appendix invariant distribution}

We recall a Liouville Theorem for degenerate elliptic equations \cite[Theorem 3]{edmunds1973liouville}. The version we present is a simplified statement (but more general than what we need).

\begin{theorem}\cite{edmunds1973liouville}\label{theorem liouville}
Consider the PDE
\begin{equation}\label{eq degenerate elliptic PDE}
\div\big(A\nabla u + au\big) + b\cdot \nabla u =0
\end{equation}
where $A:\R^m\to\R^{m\times m}$ and $a, b:\R^m\to\R^m$ are measurable functions such that
\begin{enumerate}
\item $A(\theta)$ is symmetric for all $\theta$.
\item $\lambda I \leq A \leq \Lambda I$ for functions $\lambda, \Lambda:\R^m\to (0,\infty)$ which satisfy
\[
\lambda^{-1} \in L^p_{loc}(\R^m), \qquad \Lambda \in L^q_{loc}(\R^m), \qquad \frac 1p + \frac1q < \frac2m < 1+ \frac1q.
\]
\item The function
\[
\Lambda^*(R) : R^{-\left(\frac mp + \frac mq\right)} \|\lambda^{-1}\|_{L^p(B_R)} \|\Lambda\|_{L^q(B_R)}
\]
satsifies
\[
\limsup_{R\to\infty} \Lambda^*(B_R) < \infty.
\]
\item The field $a$ satisfies the compatibility condition
\[
\int_{\R^m} a\cdot \nabla \phi\,\d\theta = 0\qquad \forall\ \phi \in C_c^\infty(\R^m).
\]
\item The function 
\[
\bar g:= A^{-1}: \big[a\otimes a + b\otimes b\big] = \big(A^{-1}\big)_{ij} \big[ a_ia_j+b_ib_j\big]
\]
satsifies the integrability condition $g\in L^q_{loc}(\R^m)$ for the same $q$ as $\Lambda$.
\item The function
\[
\bar g^*(R):= R^{2 - \frac mp- \frac mq}\|\lambda^{-1}\|_{L^p(B_R)}\|g\|_{L^q(B_R)}
\]
satisfies
\[
\limsup_{R\to\infty} \bar g^*(B_R) < \infty.
\]
\end{enumerate}
We say that $u$ is a weak solution of \eqref{eq degenerate elliptic PDE} if 
\[
\int_{B_R} \nabla u^T A\nabla u + u^2\d\theta < \infty \qquad\forall\ R>0
\]
and 
\[
\int_{\R^m} \nabla u^TA\nabla \psi + u\,a\cdot \nabla\psi + \psi \,b\cdot\nabla u \,\d\theta = 0\qquad\forall\ \psi\in C_c^\infty(\R^m).
\]
The following holds: {\em if $u$ is a weak solution of \eqref{eq degenerate elliptic PDE} and $u$ is bounded from either above or below, then $u$ is constant}.
\end{theorem}

We are now ready to prove Theorem \ref{theorem unique invariant distribution}.

\begin{proof}
We can prove that any solution $\rho$ of \eqref{eq invariant measure special} satisfies $\rho \in C^{2,\alpha}_{loc}(\R^m)$. The proof uses elliptic regularity theory and is standard, but lengthy. We refer the reader to \cite[Chapters 1-9]{gilbarg2015elliptic} for an extensive review of the methodology.

Set $u:= \rho\,f^{\frac{1+\eta\sigma}{\eta\sigma}}$ and note that 
\begin{enumerate}
\item $u\in C^2(\R^m)$ since $\rho, f\in C^2$ and $f\geq \inf f>0$,
\item $u\geq 0$ if and only if $\rho\geq 0$ and
\item $u$ is unique up to multiplication by a positive constant if and only if $\rho$ is. 
\end{enumerate}
We will use Theorem \ref{theorem liouville} to show that any non-negative solution $u$ of the stationary measure equation $\div\big(f^{-1/\eta\sigma}\nabla u\big) =0$ is constant. 

Abbreviating $\gamma:= \frac1{\eta\sigma}$, the stationary measure equation \eqref{eq invariant measure special} can be rewritten as
\begin{align*}
0 &= \div\left(f^{-\frac1{\eta\sigma}}\nabla \big(\rho\,f^{\frac{1+\eta\sigma}{\eta\sigma}}\big)\right)\\
	& = \div\big(f^{-\gamma}\nabla u\big)\\
	&= f^{-\gamma}\Delta u + \nabla(f^{-\gamma}) \cdot \nabla u\\
	&= f^{-\gamma}\left(\Delta u + \frac{\nabla f^{-\gamma}}{f^{-\gamma}}\cdot \nabla u\right)\\
	&= f^{-\gamma}\left(\Delta u + \nabla\log(f^{-\gamma})\cdot \nabla u\right)\\
	&= f^{-\gamma}\left(\Delta u - \gamma\, \nabla \big(\log f\big)\cdot \nabla u\right)
\end{align*}
on $\R^m$. Since $u\in C^2(\R^m)$ and $f>0$, we find that $u$ solves the PDE 
\begin{equation}\label{eq simplified pde}
\Delta u - \gamma\, \nabla \big(\log f\big)\cdot \nabla u
\end{equation}
on the whole space $\R^m$. In particular, $u$ is a non-negative solution of \eqref{eq simplified pde} in the sense of Theorem \ref{theorem liouville}. It is easy to see that $A$ is symmetric positive definite with $\lambda \equiv \Lambda = 1$ and that all coefficient functions are measurable. By design, the function $\Lambda^*$ is constant in $R$ if $\lambda^{-1}, \Lambda$ are constant, so in particular 
\[
\limsup_{R\to\infty} \Lambda^*(R)<\infty.
\]
This observation is independent of the choice of $p,q$.
Since $a\equiv 0$, also the compatibility condition is trivially satisfied. The only non-trivial conditions concern the function $g$, where we observe that
\[
\bar g = A^{-1} \big[a\otimes a + b\otimes b\big] =  \gamma^2\big|\nabla\log f\big|^2
\]
in our case. By assumption, $\nabla\log f$ satisfies the bound
\[
|\nabla \log f|(\theta) \leq \frac{C}{1+|\theta|}, 
\]
so
\begin{align*}
\bar g^*(R) &= R^{-\frac mq} \|1\|_{L^q(B_R)}\,R^{2-\frac mp} \|\nabla \log f\|_{L^{2p}}^2\\
	&\leq C\,R^{2-\frac mp}\left(\int_0^R \frac{r^{m-1}}{(1+r)^{2p}}\dr\right)^\frac2p\\
	&\leq C\,R^{2-\frac mp}\left(R^{m-2p}\right)^\frac2p\\
	&\leq C
\end{align*}
where $C$ is a constant whose value may change from line to line. This bound holds for any choice of  $q\in [1,\infty]$ and $p\in [1,\infty)$, so the conditions of Theorem \ref{theorem liouville} are met. Thus any non-negative solution $u$ of \eqref{eq simplified pde} is constant, meaning that $\rho = c\,f^{-\frac{1+\eta\sigma}{\eta\sigma}}$ is the only invariant measure of continuous time SGD with ML noise which has a density with respect to Lebesgue measure. 
\end{proof}

Unfortunately, the integrability condition on $\nabla \log f$ is too restrictive to be applied in the case where $f$ may vanish quadratically at $N$. We show by the way of a toy example that the proof above cannot be extended to the situation in which the objective function $f$ takes the value zero. We note however that uniqueness may hold with a different proof in the smaller class of densities $\rho$ which are non-negative and globally integrable.

\begin{example}\label{example non-uniqueness pde}
Consider the objective function $f(\theta) = \lambda\,|\theta|^k$ which has the gradient
\[
\nabla \log f(\theta) = k\,\nabla \log(|\theta|) + \nabla \log(\lambda)
	= k\,\frac{\theta}{|\theta|^2}.
\]
When we make the ansatz $u(\theta) = |\theta|^\beta$ we find that $\nabla u(\theta) = \beta\,|\theta|^{\beta-1}\,\frac{\theta}{|\theta|}$ and $\nabla \log f(\theta)\cdot \nabla u(\theta) = k\beta\,|\theta|^{\beta-2}$. By a direct calculation, we find that
\begin{align*}
\Delta u(\theta) &= \div\left(\beta\,|\theta|^{\beta-1}\,\frac{\theta}{|\theta|}\right)\\
	&= \beta\,\nabla(|\theta|^{\beta-1})\cdot \frac{\theta}{|\theta|} + \beta\,|\theta|^{\beta-1} \,\div\left(\frac\theta{|\theta|}\right)\\
	&= \beta(\beta-1)\,|\theta|^{\beta-2} + \beta\,(m-1)\,|\theta|^{\beta-2}\\
	&= \beta(\beta + m-2)\,|\theta|^{\beta-2}
\end{align*}
so $\Delta u = \gamma\nabla\log f \cdot \nabla u$ if and only if $\beta=0$ or
\[
\beta+ m-2 = k\gamma\qquad\Ra\qquad \beta = k\gamma+2 -m.
\]
In particular, there exists a non-negative and non-constant solution of \eqref{eq simplified pde}. However, note that
\[
\rho = f^{-(\gamma+1)}u = \lambda^{-(\gamma+1)}\, |\theta|^{-k\gamma-k + k\gamma+2-m} = \lambda^{-(\gamma+1)} |\theta|^{-m + 2-k}
\]
fails to be integrable at the origin if $k\geq 2$ and at infinity if $k\leq 2$. The non-integrability of $\rho$ at infinity could be healed by faster growth of $f$ (super-quadratic growth) whereas the non-integrability of $\rho$ at the origin cannot be healed for objective functions $f$ which are at least $C^2$-smooth (or, in fact, $C^{1,1}$-smooth). If the objective function is non-smooth at the global minimum and grows sufficiently quickly at $\infty$, we suspect that the invariant distribution may not be unique.
It therefore remains to be seen whether $\rho$ is the unique solution to \eqref{eq invariant measure special} which is both non-negative and integrable.
\end{example}

\section{Proof of Theorem \ref{theorem convergence}: Convergence to the invariant distribution}\label{appendix convergence}
In this section, we prove that the law $\rho_t$ of a solution of continuous time SGD with noise of the type
\[
\Sigma = \eta\sigma f\,I_{m\times m}
\]
converges to an invariant distribution, assuming that the objective landscape is of underparametrized type (i.e.\ $\inf f>0$) and the objective function grows quadratically at infinity. We note that $\rho_t$ evolves according to the PDE
\begin{align*}
\partial_t\rho &= \div\left(\eta\sigma\,f\nabla \rho + \big(1+\eta\sigma\big)\rho \nabla f\right)\\
	&= \eta \sigma\,\div\left(f^{-\frac1{\eta\sigma}}\,\nabla\left( \rho f^\frac{1+\eta\sigma}{\eta\sigma}\right)\right)
\end{align*}
which can also be re-written as an equation for $u= \rho f^\frac{1+\eta\sigma}{\eta\sigma}$ 
\begin{align*}
\partial_t u &= f^\frac{1+\eta\sigma}{\eta\sigma} \partial_t\rho\\
	&= \eta \sigma\,f^\frac{1+\eta\sigma}{\eta\sigma}\div\left(f^{-\frac1{\eta\sigma}}\,\nabla u\right)\\
	&= \eta \sigma\,\div\left(f\,\nabla u\right) - \big(1+\eta \sigma\big)\,\nabla f\cdot \nabla u\\
	&= \eta\sigma f\,\Delta u - \nabla f\cdot \nabla u.
\end{align*}
We recall the Poincar\'e-Hardy inequalities of \cite[Section 4.1]{dolbeault2012improved} and \cite{bonforte2010sharp}. Assume that there exist constants $0<\lambda \leq \Lambda$ such that
\begin{equation}\label{eq annoying growth condition}
\lambda\,(1+|\theta|^2) \leq f(\theta) \leq \Lambda\,(1+|\theta|^2)\qquad\forall\ \theta\in\R^m.
\end{equation}
For $\alpha<-\frac m2$, denote by $\mu_\alpha$ the measure which has density $(1+|\theta|^2)^\alpha$ with respect to Lebesgue measure and
\[
I_\alpha (u) = \frac1{\mu_\alpha(\R^m)} \int_{\R^m} u\,\d\mu_\alpha.
\]
Then for all $u\in C_c^\infty(\R^m)$ and $\alpha < -\frac{m+2}2$, the inequality
\begin{equation}\label{eq poincare-hardy}
\int_{\R^m} \big|u- I_{\alpha-1}(u)\big|^2\,\d\mu_{\alpha-1} \leq C(\alpha,m) \int_{\R^m} |\nabla u|^2\,\d\mu_\alpha
\end{equation}
holds where
\[
C(\alpha, m) = \begin{cases}2|\alpha| &\alpha <-m\\ 2 \big(2|\alpha|-m\big) &-m \leq \alpha < -\frac{m+2}2\\ \frac14(m-2+2\alpha)^2 &-\frac{m+2}2 < \alpha < - \frac{m-2}2\end{cases}.
\]
The inequality \eqref{eq poincare-hardy} easily extends to the Hilbert space $H_\alpha$ which is given as the closure of $C_c^\infty(\R^m)$ with the norm
\[
\|u\|_{H_\alpha}^2 = \int_{\R^m} u^2\,\d\mu_{\alpha-1} + \int_{\R^m}|\nabla u|^2\,\d\mu_\alpha.
\]
We furthermore introduce the Hilbert space $H_{f,\eta\sigma}$ which is the closure of $C_c^\infty(\R^m)$ with the norm
\[
\|u\|_{f,\eta\sigma}^2 = \int_{\R^m} u^2\,f^{-\frac1{\eta\sigma}-1} + \int_{\R^m}|\nabla u|^2\,f^{-\frac1{\eta\sigma}}\d\theta.
\]
Due to the growth condition \eqref{eq annoying growth condition}, we can identify $H_{f,\eta\sigma} = H_\alpha$ for $\alpha = -\frac1{\eta\sigma}$ with equivalent norms. We keep the second notation specifically to identify where the growth condition enters. Note that the measure $\mu_{f,\eta\sigma}$ with density $f^{-1-\frac1{\eta\sigma}}$ is finite if $\frac{1}{\eta\sigma}> \frac m2-1$, i.e.\ $\eta\sigma < \frac{2}{m-2}$.

\begin{remark}\label{remark constant functions}
We note that the constant function $f(\theta)=1$ is an element of $H_{f,\eta\sigma}$ as it can be approximated by functions $\chi_n(\theta) = \chi(\theta/n)$ where $\chi\in C_c^\infty$ is a monotone function which satisfies $0\leq \chi\leq 1$ and $\chi(\theta) = 1$ for $|\theta|\leq 1$, $\chi(\theta) = 0$ for $|\theta|\geq 2$ and $|\nabla \chi|\leq 2$. Then
\[
\lim_{n\to\infty }\chi_n(\theta) = 1
\]
pointwise and $0\leq f_n\leq 1$, so convergence holds in $L^2(\mu_{f,\eta\sigma})$ by the dominated convergence theorem. Furthermore
\[
\int_{\R^m} |\nabla \chi_n|^2\,f^{-\frac1{\eta\sigma}}\,\d\theta \leq 4 \int_{\{n\leq |\theta|\leq 2n\}} n^{-2} f^{-\frac1{\eta\sigma}}\d\theta \leq 4\lambda^{-1} \int_{\{n\leq |\theta|\leq 2n\}} f^{-1-\frac1{\eta\sigma}}\d\theta.
\]
Since the measure $\mu_{f,\eta\sigma}$ is finite, we deduce that
\[
\lim_{n\to\infty}\int_{\R^m} |\nabla \chi_n|^2\,f^{-\frac1{2\eta\sigma}}\,\d\theta = \lim_{n\to\infty} \mu_{f,\eta\sigma}(\{n\leq |\theta|\leq 2n\}) = 0.
\]
\end{remark}

In the proof in Remark \ref{remark constant functions}, we used that $f$ grows quadratically at infinity. 
 The results in \cite[Section 4.1]{dolbeault2012improved} also apply more generally in the case where the measure $\mu_\alpha$ is not finite and generally $\alpha \neq \alpha^*= - \frac{m-2}2$. In this case, constant functions are not elements of $H_\alpha$, and the average integral is replaced by $I_\alpha(f) =0$.

A main reason in choosing the $H_{f,\eta\sigma}$-norm over the $H_\alpha$ norm in the analysis is the following observation.

\begin{lemma}\label{lemma closeable operator}
Let 
\[
\mathrm{Dom}(A) = \big\{u \in H^2_{loc}(\R^m) \cap H_{f,\eta\sigma} : \eta\sigma f\,\Delta u - \nabla f\cdot \nabla u \in L^2(\mu_{f,\eta\sigma})\big\}.
\]
The operator 
\[
A: \mathrm{Dom}(A)\to L^2(\mu_{f,\eta\sigma}), \qquad A u = f^{1+\frac1{\eta\sigma}} \,\div\big(f^{-\frac1{\eta\sigma}}\nabla u\big) 
\]
is maximal monotone, closeable and self-adjoint on $L^2(\mu_{f,\eta\sigma})$ and satisfies
\[
\langle Au, v\rangle_{L^2(\mu_{f,\eta\sigma})} = - \int_{\R^m} \nabla u\cdot \nabla v\,f^{-\frac1{\eta\sigma}}\,\d\theta \qquad\forall\ u,v \in \mathrm{Dom}(A).
\]
\end{lemma}

\begin{proof}
{\bf Integration by parts.} Use the same bump function as in Remark \ref{remark constant functions} and see that  
\begin{align*}
\langle Au, \chi_nv\rangle_{L^2(\mu_{f,\eta\sigma})} &= \int_{\R^m} \chi_n v\,f^{1+\frac1{\eta\sigma}}\div\big(f^{-\frac1{\eta\sigma}}\nabla u\big)\,f^{-1-\frac1{\eta\sigma}}\,\d\theta\\
	&= \int_{\R^m} \chi_n v\,\div\big(f^{-\frac1{\eta\sigma}}\nabla u\big)\,\d\theta\\
	&= - \int_{\R^m} \big(v\nabla \chi_n +\chi_n\nabla v\big)\cdot f^{-\frac1{\eta\sigma}}\nabla u\,\d\theta\\
	&= - \int_{\R^m}\chi_n\,f^{-\frac1{\eta\sigma}}\nabla u\cdot \nabla v,\d\theta+ \int_{\R^m} f^{-\frac1{\eta\sigma}}u \nabla\chi_n\cdot \nabla u\,\d\theta.
\end{align*}
By the monotone convergence theorem, we find that $\chi_n u \to u$ in $L^2(\mu_{f,\eta\sigma})$ and 
\[
\lim_{n\to\infty} \int_{\R^m}\chi_n\,f^{-\frac1{\eta\sigma}}\nabla u\cdot \nabla v\,\d\theta = \int_{\R^m}f^{-\frac1{\eta\sigma}}\nabla u\cdot \nabla v\,\d\theta.
\]
The boundary integral is bounded by
\begin{align*}
\left|\int_{\R^m} f^{-\frac1{\eta\sigma}}v\,\nabla\chi_n\cdot \nabla u\,\d\theta\right| &\leq \left(\int_{\{n\leq |\theta|\leq 2n\}} f^{-\frac1{\eta\sigma}}|\nabla\chi_n|^2v^2\,\d\theta\right)^\frac12\left(\int_{\{n\leq |\theta|\leq 2n\}} f^{-\frac1{\eta\sigma}}|\nabla u|^2\,\d\theta\right)^\frac12\\
	&\leq C \left(\int_{\{n\leq |\theta|\leq 2n\}} f^{-1-\frac1{\eta\sigma}}v^2\,\d\theta\right)^\frac12\left(\int_{\{n\leq |\theta|\leq 2n\}} f^{-\frac1{\eta\sigma}}|\nabla u|^2\,\d\theta\right)^\frac12
\end{align*}
since $f$ grows quadratically at infinity. Since both integrals are globally finite, the right hand side converges to $0$ as $n\to\infty$. Thus
\[
\langle Au, v\rangle_{L^2(\mu_{f,\eta\sigma})} = \lim_{n\to\infty}\langle Au, \chi_nv\rangle_{L^2(\mu_{f,\eta\sigma})} = - \int_{\R^m}f^{-\frac1{\eta\sigma}}\nabla u\cdot \nabla v\,\d\theta.
\]
As a consequence, we find in particular that
\[
\int_{\R^m}|\nabla u|^2\,f^{-\frac1{\eta\sigma}}\,\d\theta = \langle Au, u\rangle_{L^2(\mu_{f,\eta\sigma})} \leq \|Au\|_{L^2(\mu_{f,\eta\sigma}} \|u\|_{L^2(\mu_{f,\eta\sigma})}.
\]

{\bf Operator properties.} To apply the Hardy-Poincar\'e inequality \ref{eq poincare-hardy}, let $\alpha = - \frac1{\eta\sigma}$. Then
\begin{align*}
\|u- \langle u\rangle_{\mu_{f,\eta\sigma}}\|_{L^2(\mu_{f,\eta\sigma})}^2
	&\leq \|u- \langle u\rangle_{\mu_{\alpha-1}}\|_{L^2(\mu_{f,\eta\sigma})}^2\\
	&\leq C(\alpha,m)\,\lambda^{-\frac1{\eta\sigma}-1}  \|u- \langle u\rangle_{\mu_{\alpha-1}}\|_{L^2(\mu_{\alpha-1})}^2\\
	&\leq C(\alpha,m)\,\lambda^{-\frac1{\eta\sigma}-1}  \|\nabla u\|_{L^2(\mu_{\alpha})}^2\\
	&\leq C\left((\eta\sigma)^{-1},m\right) \Lambda^\frac1{\eta\sigma} \lambda^{-1-\frac1{\eta\sigma}} \int_{\R^m} |\nabla u|^2\,f^{-\frac1{\eta\sigma}}\,\d\theta.
\end{align*}
The proof is now standard and mimics \cite[Theorem 10.2]{MR2759829}. In particular,
\[
[u]_{H_{f,\eta\sigma}} = \int_{\R^m}|\nabla u|^2\,f^{-\frac1{\eta\sigma}}\,\d\theta
\]
is a semi-norm on $H_{f,\eta\sigma}$. We find by the Lax-Milgram theorem \cite[Corollary 5.8 and Theorem 9.25]{MR2759829} that for every $v\in L^2(\mu_{f,\eta\sigma})$ there exists a unique solution $u\in H_{f,\eta\sigma}$ of the equation $(I+A) u = v$ in the weak sense, i.e.
\[
\int_{\R^m} u\phi \,f^{-1-\frac1{\eta\sigma}} \,\d\theta+ \int_{\R^m} \nabla u\cdot \nabla \phi\,f^{-\frac1{\eta\sigma}}\,\d\theta = \int_{\R^m} u\phi \,f^{-1-\frac1{\eta\sigma}} \,\d\theta
\]
for all $\phi \in H_{f,\eta\sigma}$. Since $f\in C^2(\R^m)$ is strictly positive, we can apply elliptic regularity theory as in \cite[Theorem 8.8]{gilbarg2015elliptic} to show that in fact $u\in H^2_{loc}(\R^m)$, i.e.\ $u\in \mathrm{Dom}(A)$. 

Since $A$ is maximally monotone and symmetric, it is densely defined, closeable and self-adjoint \cite[Proposition 7.1]{MR2759829}.
\end{proof}

The proof of the Hardy-Poincar\`e inequality \eqref{eq poincare-hardy} with optimal constant involves the spectral analysis of the operator $A$ in the case that $f(\theta) = 1+|\theta|^2$. The spectrum is easier to analyze in the special case due to radial symmetry.

\begin{corollary}\label{corollary mild solution}
Let $u^0\in L^2(\mu_{f,\eta\sigma})$. Then there exists a mild solution $u$ of 
\[
\begin{pde}
\dot u &= f^{1+\frac1{\eta\sigma}} \div(f^{-\frac1{\eta\sigma}}\nabla u)&t>0\\
u &= u^0 &t=0
\end{pde}
\]
and the weighted average
\[
\langle u\rangle_{\mu_{f,\eta\sigma}} = \frac1{\mu_{f,\eta\sigma}(\R^m)} \int_{\R^m} u\,\d\mu_{f,\eta\sigma}
\]
is constant in time. Furthermore
\[
\left\| u(t)  - \langle u^0\rangle_{\mu_{f,\eta\sigma}}\right\|_{L^2(\mu_{f,\eta\sigma})} \leq e^{-\nu t} \left\| u^0  - \langle u^0\rangle_{\mu_{f,\eta\sigma}}\right\| _{L^2(\mu_{f,\eta\sigma})}
\]
for some $\nu$ depending on $\lambda,\Lambda,m$ and $\eta\sigma$.
\end{corollary}

\begin{proof}
The proof follows exactly as in \cite[Section 7.4.a]{MR2597943} or \cite[Section 10.1]{MR2759829}. The correct choice of function space $H_{f,\eta\sigma}$ reduces the analysis of a non-divergence form operator to that of a divergence form operator. We note that
\[\showlabel \label{eq average constant}
\frac{d}{dt} \int_{\R^m}u\,f^{-1-\frac1{\eta\sigma}}\,\d\theta = \frac{d}{dt} \langle u, 1\rangle_{L^2(\mu_{f,\eta\sigma})} = \langle Au, 1\rangle_{L^2(\mu_{f,\eta\sigma})} = 0
\]
since $\nabla 1\equiv 0$. Furthermore
\begin{align*}
\frac{d}{dt} \|u- \langle u\rangle_{\mu_{f,\eta\sigma}}\|_{L^2(\mu_{f,\eta\sigma})}^2 &= - \int_{\R^m} |\nabla u|^2\,f^{-\frac1{\eta\sigma}}\,\d\theta\\
	&\leq  C\left((\eta\sigma)^{-1},m\right)\Lambda^\frac1{\eta\sigma} \lambda^{-1-\frac1{\eta\sigma}}\|u- \langle u\rangle_{\mu_{f,\eta\sigma}}\|_{L^2(\mu_{f,\eta\sigma})}^2.
\end{align*}
Exponential decay follows by applying Gr\"onwall's inequality. It remains to show that $u\geq 0$. We apply Stampacchia's truncation method as outlined e.g.\ in \cite[Theorem 10.3]{MR2759829}. Namely, let $G:\R\to\R$ be a convex $C^2$-function such that
\begin{enumerate}
\item $G(z) = 0$ for all $z\leq 0$ and
\item $0\leq G'(z) \leq 1$.
\end{enumerate}
Since $u\in C^1\big((0,\infty), H^1_{f,\eta\sigma})$ and $G$ grows only linearly, we may compute
\begin{align*}
\frac{d}{dt} \int_{\R^m} G(-u)\,\d\theta &= - \int_{\R^m}G'(-u)\,\partial_tu\,\d\theta\\
	&= - \int_{\R^m}G'(-u)\,Au\,\d\theta\\
	&= \int_{\R^m} \nabla\big(G'(-u)\big)\cdot \nabla u\,f^{-\frac1{\eta\sigma}}\,\d\theta\\
	&= - \int_{\R^m} G''(u)\,|\nabla u|^2\,\d\theta.
\end{align*}
The right hand side is non-positive since $G''\geq 0$. Thus, if $u^0\geq 0$, we have
\[
\int_{\R^m} G(-u(0))\,\d\theta= 0 , \qquad \frac{d}{dt} \int_{\R^m} G(-u(t))\,\d\theta\leq 0,
\]
meaning that $u\geq 0$ for all $t\geq 0$ and Lebesgue-almost everywhere.
\end{proof}

Theorem \ref{theorem convergence} now follows from Corollary \ref{corollary mild solution}. It remains to reconstruct $\rho$ from $u$. The notion of solution we have constructed here is stronger than what we required of $\rho$ before, since we have shown that (at least for positive times) $\rho$ has two spatial derivatives in $L^2_{loc}(\R^m)$. It is therefore not even necessary to put both derivatives on the test function.

\begin{proof}[Proof of Theorem \ref{theorem convergence}]
Let $\phi\in C_c^\infty(\R^m)$ and set $\rho = u\,f^{-1-\frac1{\eta\sigma}}$. We note that, due to \eqref{eq average constant}, the integral
\[
\int_{\R^m}\rho\,\d\theta = \int_{\R^m}u\,f^{-1-\frac1{\eta\sigma}}\,\d\theta \equiv 1
\]
is constant in time. To show that $\rho$ is a probability density, it remains to show that $\rho\geq 0$, which is true since $u\geq 0$.
Furthermore
\begin{align*}
\frac{d}{dt} \langle \rho, \phi\rangle_{L^2(\R^m)} &= \frac{d}{dt} \langle u, \phi\rangle_{L^2(\mu_{f,\eta\sigma})}
	= \langle Au, \phi\rangle_{L^2(\mu_{f,\eta\sigma})}
	= \langle u, A\phi\rangle_{L^2(\mu_{f,\eta\sigma})}
	= \langle \rho, A\phi\rangle_{L^2(\R^m)}.
\end{align*}
\end{proof}

For an extension to the case where $f$ may be mildly degenerate, see Theorem \ref{theorem convergence appendix} in Appendix \ref{appendix poincare}.

\section{Proofs concerning flat minimum selection}\label{appendix flat minimum}

We first establish the global and local integrability of the invariant distributions.

\begin{proof}[Proof of Lemma \ref{lemma integrability}]
{\bf Local integrability.}
Since $N = \{f=0\}$ and $f$ grows at infinity, we see that $f^\alpha$ is bounded on $\R^m \setminus U$ where $U$ is a neighbourhood of $N$. To check whether $f^\alpha$ is locally integrable at $\bar \theta \in N$, we may change coordinates such that $\bar \theta=0$ and $N\cap B_\eps(\bar \theta) = \{ \theta: \theta_{m-n+1} = \dots \theta_m =0\}$. Lebesgue measure before and after the change of coordinates have bounded densities with respect to each other, so the integrability of $f^\alpha$ is not affected. 

Since $D^2f(\bar \theta)$ has full rank $m-n$ and $\nabla f(\bar \theta) =0$, there exists a neighbourhood of $\bar \theta$ and constants $c, \widetilde c, C, \widetilde C$ such that the estimate
\begin{align*}
{\widetilde c}\,\big|(\theta_1,\dots,\theta_{n-m}, 0,\dots, 0)\big|^2 &\leq \frac{c}2 (\theta- \bar \theta)^T\,D^2f(\bar \theta)\,(\theta-\bar \theta) \leq f(\theta) \\
	&\hspace{2cm}\leq \frac{C}2 (\theta- \bar \theta)^T\,D^2f(\bar \theta)\,(\theta-\bar \theta)\leq \widetilde C\,\big|(\theta_1,\dots,\theta_{n-m}, 0,\dots,0)\big|^2
\end{align*}
is valid for all $x$ in the neighborhood. Thus $f^\alpha$ is integrable at $\bar \theta$ if and only if 
\[
g_\alpha(z) = |z|^{2\alpha}
\]
is integrable at the origin in $\R^{m-n}$, i.e.\ if and only if $2\alpha + m-n-1 > -1$.

{\bf Global integrability.} If $f^\alpha$ is locally integrable, then it is globally integrable if and only if it is integrable on the set $\R^m\setminus B_R(\bar \theta)$. Since $f(\theta) \geq c_1\,|x|^\gamma$ if $R$ is large enough, this follows if
\[
h_\alpha(z) = |z|^{\gamma\alpha}
\]
is integrable at infinity in $\R^m$, i.e.\ if and only if $\gamma\alpha +m-1 < -1$.
\end{proof}

We establish the behavior of continuous time SGD with ML noise at the critical noise threshold for objective functions which mimic the overparametrized regime.

\begin{proof}[Proof of Theorem \ref{theorem critical flatness}]
{\bf Step 1.} First we demonstrate that a limit exists. Note that  
\[
\int_{\R^m\setminus B_R(0)} f^\alpha(\theta)\d\theta \leq c_1^\alpha \int_{\R^m\setminus B_R(0)} |\theta|^{\gamma\alpha}\d\theta \to c_1^{\alpha^*} \int_{\R^m\setminus B_R(0)} |\theta|^{\gamma\alpha^*}\d\theta <\infty
\]
as $\alpha\to \alpha^*$. The finiteness follows from the fact that
\[
-\gamma \alpha^* = \gamma \frac{m-n}2 > \frac{m-n}2\,\frac{2m}{m-n} = m.
\]
On the other hand, as in the proof of Theorem \ref{theorem integrability}, we see that 
\[
\lim_{\alpha\to \alpha^*} \int_{B_R(0)} f^\alpha(\theta)\d\theta = +\infty.
\]
Thus the sequence of Radon measures
\[
\pi_\alpha:= \frac{f^\alpha}{\int_{\R^m}f^\alpha(\theta)\d\theta} \cdot \d\theta
\]
is tight. By Prokhorov's Theorem \cite[Satz 13.29]{klenke2006wahrscheinlichkeitstheorie}, there exists a weakly convergent subsequence to a limiting probability measure $\pi^*$ on $\overline{B_R(0)}$. We further note that 
\[
\int_{B_R(0)\cap \{f>\eps\}} f^\alpha\d\theta \leq \big|B_R(0)\big|\,\eps^\alpha \leq \big|B_R(0)\big|\,\eps^{\alpha^*} <\infty,
\]
so $\pi^*(\{f>\eps\}) =0$ for all $\eps>0$. We conclude that $\pi^*(\R^m\setminus N) =0$. Since $N$ is closed, this means that $\pi^*$ is supported on $N$.

{\bf Step 2.} In this step, we simplify the geometry of the problem. Let $\theta_0\in N$. There exists a $C^2$-diffeomorphism 
\[
\phi:U\to B_r(0)
\]
such that 
\begin{itemize}
\item $U$ is a convex neighbourhood of the origin,
\item $\phi(0) = \theta_0$,
\item $O:= D\phi(0)$ is a rotation, and
\item $\phi^{-1} (N\cap B_r(\theta_0)) = \{\theta_{n+1} = \dots = \theta_m =0\}\cap U$. 
\end{itemize}
Note that
\[
\partial_i(f\circ\phi) = (\partial_i\phi^k)\,(\partial_kf)\circ \phi, \qquad \partial_i\partial_j(f\circ\phi) = (\partial_i\phi^k)(\partial_j\phi^l)\,(\partial_k\partial_lf) + (\partial_i\partial_j\phi^k)\,(\partial_kf)\circ\phi.
\]
Since $\nabla f=0$ on $N$, we see that 
\[
D^2(f\circ\phi)(0) = O^T\,D^2f(\theta_0) O.
\]
In particular, $D^2(f\circ\phi)(0)$ and $D^2f(\theta_0)$ are symmetric matrices with identical eigenvalues. We note the following:

\begin{enumerate}
\item Since $D^2f\circ \phi$ is continuous and positive definite in the directions orthogonal to $\{\theta_1=\dots=\theta_n=0\}$, for every $\mu>0$ there exists $r$ such that 
\[
(1-\mu)\,D^2(f\circ\phi)(0) \leq D^2(f\circ\phi)(\theta_1,\dots,\theta_n,0,\dots, 0) \leq (1+\mu) D^2(f\circ\phi)(0) \qquad \forall\ |(\theta_1,\dots,\theta_n)|<r.
\]
\item Since $\phi \in C^1$ and $D\phi(0)\in O(n)$, for every $\mu>0$ there exists $r>0$ such that
\[
\phi\big(B_{(1-\mu)r}(0)\big) \subseteq B_r(\theta_0)\subseteq \phi\big(B_{(1+\mu)r}(0)\big).
\]
\item Since $\phi \in C^1$ and $D\phi(0)\in O(n)$, for every $\mu>0$ there exists $r>0$ such that
\[
1-\mu \leq |\det D\phi(\theta)|\leq 1+\mu\qquad \forall\ \theta \in B_r(0).
\]
\end{enumerate}

Combining the observations, we find that
\[
(1-\mu)\,\frac12\,\theta^TD^2(f\circ\phi)(0)\theta \leq (f\circ\phi)(\theta) \leq (1+\mu)\,\frac12\,\theta^TD^2(f\circ\phi)(0)\,\theta\qquad \forall\ \theta \in B_r(0)
\]
and in particular
\begin{align*}
\lim_{r\to 0} \frac{\pi^*(B_r(\theta_0))}{\omega_nr^n} &= \lim_{r\to 0} \lim_{\alpha\to \alpha^*} \frac1{c_\alpha\,\omega_nr^n} \int_{B_r(\theta_0)} f^\alpha(\theta)\,\d\theta\\
	& = \lim_{r\to 0} \lim_{\alpha\to \alpha^*} \frac1{c_\alpha\,\omega_nr^n} \int_{B_r(0)} \left(\frac12\,\theta^T D^2(f\circ\phi)(0)\theta\right)^\alpha \d\theta.
\end{align*}
From now on, without loss of generality we will assume that $\theta_0=0$ and that $N$ is flat at $\theta_0$. 

{\bf Step 3.} In this step, we finally compute the density. Denote by $\H^\beta$ the $\beta$-dimensional Hausdorff measure (i.e.\ the natural area measure on $S^{m-n-1}$ and $N$ respectively for appropriate values of $\beta$). We denote by $\widehat{B^n_r}$ and $\widehat{B^{m-n}_r}$ the balls of radius $r$ around the origin in $\R^n$ and $\R^{m-n}$ respectively and decompose $\theta= \theta_N + \theta^\bot$ where $\theta_N = (\theta_1,\dots,\theta_n, 0,\dots, 0)$ and $\theta^\bot = (0, \dots, 0,\theta_{n+1}, \dots, \theta_m)$. Then
\begin{align*}
\int_{\widehat{B^{m-n}_r}} \left(\frac12(\theta^\bot)^T D^2f(0)\theta^\bot\right)^\alpha \d\theta^\bot
	 &= \int_0^r \int_{S^{m-n-1}}\left(\frac12(s\nu)^T \widehat {D^2f(0)}(s\nu)\right)^\alpha\,s^{m-n-1}\d\H^{m-n-1}_\nu\ds\\
	 &=2^{-\alpha}\left(\int_0^r s^{2\alpha+n-1}\ds\right)\left(\int_{S^{n-1}}\left(\nu^T\widehat {D^2f(0)}\nu\right)^\alpha\d\H^{m-n-1}_\nu\right)\\
	 &= \frac{1}{2^\alpha(2\alpha+m-n)} r^{2\alpha+m-n} \int_{S^{n-1}}\left(\nu^T\widehat {D^2f(0)}\nu\right)^\alpha\d\H^{m-n-1}_\nu.
\end{align*}
We consider a normalized density
\[
\tilde\rho_\alpha = (2\alpha+m-n)\,f^\alpha \approx (2\alpha+m-n)\,\theta^TD^2f(0)\,\theta
\]
and note that
\[
\pi^* = \lim_{\alpha\searrow \alpha^*} \rho_\alpha\cdot \d\theta = \lim_{\alpha\searrow \alpha^*} \frac{\tilde\rho_\alpha}{\int_{\R^m}\tilde\rho_\alpha(\theta')\,\d\theta'}\cdot\d\theta.
\]
We finally compute
\begin{align*}
(2\alpha+m-n)\int_{B_r(0)}& \left(\frac12(\theta^\bot)^T D^2f(0)\,\theta^\bot\right)^\alpha \d\theta\\
	&= \int_{\widehat B^n_r} \int_{\widehat{B^{m-n}_{\sqrt{r^2-|\theta_N|^2}}}} \left(\frac12(\theta^\bot)^T D^2f(0)\,\theta^\bot\right)^\alpha \d\theta^\bot\d\theta_N\\
	&= \frac{1}{2^\alpha} \int_{\widehat B^n_r}\big(r^2-|\theta_N|^2\big)^\frac{2\alpha+m-n}2\d\theta_N\,\int_{S^{n-1}}\left(\nu^T\widehat {D^2f(0)}\nu\right)^\alpha\d\H^{n-1}_\nu\\
	&= \frac{n\omega_n}{2^\alpha}\int_0^r \big(r^2-s^2\big)^\frac{2\alpha+m-n}2 s^{n-1}\ds \int_{S^{n-1}}\left(\nu^T\widehat {D^2f(0)}\nu\right)^\alpha\d\H^{n-1}_\nu\\
	&\to \frac{\omega_n r^n}{2^{\alpha^*}}\int_{S^{n-1}}\left(\nu^T\widehat {D^2f(0)}\nu\right)^\alpha\d\H^{n-1}_\nu.
\end{align*}
as $\alpha\searrow\alpha^*$. As a consequence
\begin{align*}
\lim_{r\to 0} \frac{\pi^*(B_r(\theta_0))}{\omega_nr^n} &= \lim_{r\to 0} \frac{\int_{S^{n-1}}\left(\nu^T\widehat {D^2f(\theta_0)}\nu\right)^\alpha\d\H^{m-n-1}_\nu}{\int_{N}\int_{S^{n-1}}\left(\nu^T\widehat {D^2f(\theta)}\nu\right)^\alpha\d\H^{m-n-1}_\nu\d\H^n_\theta}.
\end{align*}
We conclude that $\pi^*$ is absolutely continuous with respect to $\H^n|_N$ by with density $\rho^*$ by a Theorem of Marstrand \cite[Theorem 6.8]{de2006lecture}.
\end{proof}

Next, we sketch the behavior of continuous time SGD with homogenous noise.

\begin{proof}[Proof of Theorem \ref{theorem flat minima homogeneous}]
{\bf Step 1.} In this step, we prove tightness of the invariant distributions. Since
\begin{align*}
\int_{\R^m\setminus B_R(0)} \exp\left(-\frac{f(\theta)}{\eta}\right)\d\theta &\leq \int_{\R^m\setminus B_R(0)} \exp\left(-\frac{c_1|\theta|^\kappa}{\eta}\right)\d\theta\\
	&= m\omega_m \int_R^\infty \exp\left(- \left( c_1^{1/\kappa}\eta^{-1/\kappa}r\right)^\kappa\right)r^{m-1}\dr\\
	&= \frac{\eta^{m/\kappa}}{c_1^{m/\kappa}} \int_{c_1^{1/\kappa}\eta^{-1/\kappa}R} s^{m-1}e^{-s^\kappa}\ds\\
\int_{B_R(0)\cap \{f\geq\eps\}} \exp\left(-\frac{f(\theta)}{\eta}\right)\d\theta&\leq \omega_mR^m\,\exp\left(-\frac\eps\eta\right).
\end{align*}
On the other hand, choose $\bar\theta\in N$. Since $f\in C^2$, the function $f$ is locally Lipschitz, so there exists $c>0$ such that $f(\theta)<\eta$ for all $\theta \in B_{c\eta}(\bar\theta)$. Thus
\begin{align*}
\int_{B_{c\eta}(\bar\theta)} \exp\left(-\frac{f(\theta)}{\eta}\right)\d\theta &\geq\int_{B_{c\eta}(\bar\theta)}\exp\left(-\frac{\eta}{\eta}\right)\d\theta\\
	&= \omega_mc^m\,e^{-1}\,\eta^{m}.
\end{align*}
In particular 
\begin{align*}
\pi^*\big(\{f\geq \eps\}\big) &\leq \liminf_{\eta\to0} \frac{\int_{\R^m\setminus B_R(0)}\exp\left(-\frac{f(\theta)}{\eta}\right)\d\theta + \int_{B_R(0)\cap \{f\geq\eps\}} \exp\left(-\frac{f(\theta)}{\eta}\right)\d\theta}{\omega_mc^m\,e^{-1}\,\eta^{m}}\\
	&\leq C \liminf_{\eta\to0} \frac{\eta^{m/\kappa} \int_{\tilde R\eta^{1/\kappa}}s^{m-1} e^{-s^\kappa}\ds + \exp(-\eps/\eta)}{\eta^m}\\
	&=0.
\end{align*}
The last step follows immediately if $\kappa\in(0,1)$ and with a slight effort if $\kappa>1$. As before, we conclude that the measures $\pi_\eta:= \rho_\eta\cdot \d\theta$ are tight and that a limiting measure $\pi^*$ exists, which is supported on $N$.

{\bf Step 2.} We apply the same geometric simplification as in the proof of Theorem \ref{theorem critical flatness}.

{\bf Step 3.} As before, we compute the density by
\begin{align*}
\int_{B_r(0)} \exp\left(-\frac{f(\theta)}\eta\right)\d\theta &\approx\int_{B_r(0)}\exp\left(-\frac{\theta^T\,D^2f(0)\,\theta}{2{\eta}}\right)\d\theta\\
	&= \int_{\widehat{B_r^n(0)}} \int_{\widehat{B_{\sqrt{r^2-|\theta_N|^2}}^{m-n}(0)}}\exp\left(-\frac{\theta^T\,D^2f(0)\,\theta}{2{\eta}}\right)\d\theta^\bot\d\theta_N.
\end{align*}
The reduced Hessian has a unique symmetric positive definite square root $\sqrt{D^2f(0)}$ and $\det \sqrt{D^2f(0)} = \sqrt{\det D^2f(0)}$. We observe that
\begin{align*}
\lim_{\eta\to0} \eta^{-\frac{m-n}2}&\int_{\widehat{B_{s}^{m-n}(0)}}\exp\left(-\frac{\theta^T\,D^2f(0)\,\theta}{2{\eta}}\right)\d\theta^\bot\\
	&= \lim_{\eta\to0} \frac1{\sqrt{\det(D^2f(0))}} \int_{\widehat{B_{s}^{m-n}(0)}}\exp\left(-\frac{(\eta^{-1/2}\sqrt{D^2f(0)}\theta)^T(\eta^{-1/2}\sqrt{D^2f(0)}\theta)}{2}\right)\\
		&\hspace{7cm}\eta^{-\frac{m-n}2}\,\det(\sqrt{D^2f(0)})\,\d\theta^\bot\\
	&= \frac1{\sqrt{\det(D^2f(0))}} \lim_{\eta\to0}\int_{\eta^{-1/2}\sqrt{D^2f(0)}\big(\widehat{B_{s}^{m-n}(0)}\big)}\exp\left(-\frac{|z|^2}2\right)\\
	&\to \frac1{\sqrt{\det(D^2f(0))}} (2\pi)^{\frac{m-n}2}
\end{align*}
as $\eta\to0$, hence
\begin{align*}
\lim_{r\to0} \frac{\pi^*(B_r(0))}{\omega_nr^n} &= c \lim_{r\to0}\lim_{\eta\to0}\eta^{-\frac{m-n}2} \frac1{\omega_nr^n} \int_{\widehat{B_r^n(0)}} \int_{\widehat{B_{\sqrt{r^2-|\theta_N|^2}}^{m-n}(0)}}\exp\left(-\frac{\theta^T\,D^2f(0)\,\theta}{2{\eta}}\right)\d\theta^\bot\d\theta_N\\
	&= c \lim_{r\to0} \frac1{\omega_nr^n} \int_{\widehat{B_r^n(0)}} \lim_{\eta\to0}\eta^{-\frac{m-n}2} \int_{\widehat{B_{\sqrt{r^2-|\theta_N|^2}}^{m-n}(0)}}\exp\left(-\frac{\theta^T\,D^2f(0)\,\theta}{2{\eta}}\right)\d\theta^\bot\d\theta_N\\
	&= c \lim_{r\to0} \frac1{\omega_nr^n} \int_{\widehat{B_r^n(0)}} \frac1{\sqrt{\det(D^2f(0))}} \d\theta_N\\
	&= c \frac1{\sqrt{\det(D^2f(0))}}
\end{align*}
where $c$ is a constant which may change value from line to line.
\end{proof}

Finally, we consider flat minimum selection for underparametrized machine learning models. This case is similar to homogeneous noise in that there is no singularity at $N$, so that the reweighting does not compensate divergence.

\begin{proof}[Proof of Theorem \ref{theorem flatness ml underparametrized}]
{\bf Step 1.} Without loss of generality we can assume that $c_1R^\kappa \geq 2\eps$ (otherwise we can choose $R$ larger without violating the condition. Compactness follows as above by the estimates
\begin{align*}
\int_{\R^m\setminus B_R(0)} f^\alpha(\theta)\d\theta &\leq \int_{\R^m\setminus B_R(one0)}\big(c_1|\theta|^\kappa\big)^\alpha\d\theta\\
	&= \omega_m\,c_1^\alpha \int_R^\infty r^{\kappa\alpha+m-1}\dr\\
	&= \omega_m c_1^\alpha\frac{R^{\kappa\alpha+m}}{|\kappa\alpha+m|}\\
\int_{B_R(0) \cap \{f\geq (1+\mu)\eps\}} f^\alpha(\theta)\d\theta &\leq (1+\mu)^\alpha \omega_m R^m\,\eps^\alpha\\
\int_{\{f\leq (1+\mu/2)\eps\}} f^\alpha(\theta)\d\theta &\geq \big|\{f \leq 1+\mu/2)\eps\}\big| \left(1+\frac\mu2\right)^\alpha\eps^\alpha
\end{align*}
for $\mu>0$. Thus
\begin{align*}
\frac{\int_{\{f\geq(1+\mu)\eps\}} f^\alpha(\theta)\d\theta}{\int_{\R^m} f^\alpha(\theta)\d\theta} &\leq C\frac{(c_1R^\kappa)^\alpha + (1+\mu)^\alpha\eps^\alpha}{\big(1+\frac\mu2\big)^\alpha\eps^\alpha}\\
	&\leq C\left\{\left(\frac{2\eps}{\big(1+\frac\mu2\big)\eps}\right)^\alpha + \left(\frac{1+\mu}{1+\frac\mu2}\right)^\alpha\right\}\\
	&\to 0
\end{align*}
for $\mu<2$. In particular, the sequence of measures $\pi_\alpha$ is tight and the limiting measure $\pi^*$ is supported on $N$.

{\bf Step 2.} Again, we simplify the geometry of $f$ and $N$ as in Theorem \ref{theorem critical flatness}.

{\bf Step 3.} We approximate 
\[
f(\theta)\approx f(0) + \frac12\theta^T D^2f(0)\,\theta = \eps + \frac12\theta^T D^2f(0)\,\theta
\]
with the same justification as before and note that
\begin{align*}
\eps^{-\alpha}\int_{\widehat {B^{m-n}_r(0)}} & \left(\eps + \frac12\theta^T D^2f(0)\,\theta\right)^\alpha\d\theta
	= \int_{\widehat {B^{m-n}_r(0)}} \left(1 + \frac1{2\eps}\theta^T D^2f(0)\,\theta\right)^\alpha\d\theta\\
	&=\sqrt{\det D^2f(0)} \int_{\sqrt{2}\eps^{-1/2} \sqrt{D^2f(0)}\big(\widehat {B^{m-n}_r(0)}\big)} \left(1 + |z|^2\right)^\alpha\dz.
\end{align*}
It is straight-forward to see that 
\[
\lim_{\alpha\to-\infty} \frac{\int_{U} (1+|z|^2)^\alpha\dz} {\int_{\R^m} (1+|z|^2)^\alpha\dz} = 1
\]
for every neighbourhood of the origin $U$. We can conclude the argument as in Theorem \ref{theorem flat minima homogeneous}.
\end{proof}

\section{A Poincar\'e-Wirtinger-Hardy inequality and applications}\label{appendix poincare}

In this appendix, we prove that solutions to the continuous time SGD evolution equation converge to the invariant measure in certain toymodels for the overparametrized regime. The assumptions we make are restrictive and exclude objective functions which are $C^{1,1}$-smooth (in particular, $C^2$-smooth) at the minimum, or which vanish on a manifold of positive dimension. Furthermore, we require the very specific noise scaling $1+ \frac1{\eta\sigma} = \frac{m}2$. The reason for these restrictions will become clearer below. 

\begin{theorem}\label{theorem convergence appendix}
Let $f:\R^m\to [0,\infty)$ be a $C^1$-function such that there exists a finite set $\Theta =\{\bar\theta_1,\dots, \bar\theta_n\}$ and constants $0<c_1\leq C_1$ such that 
\[
c_1\leq\frac{f(\theta)}{\dist^2(\theta, \Theta)\,\big[|\log|+1\big]^2 (\theta, \Theta)}\leq C_1.
\]
Assume that $-\frac1{\eta\sigma}-1 = -\frac{m}2$ (in particular, $m\geq 3$). Let $\rho_0$ be a probability density on $\R^m$ such that
\[
\int_{\R^m} \rho_0^2\,f^{1+\frac1{\eta\sigma}}\,\d\theta < \infty.
\]
Then there exists a solution $\rho$ of the equation 
\[
\partial_t\rho = \div\left( \eta\sigma f\,\nabla \rho + (1+\eta\sigma)\rho\,\nabla f\right)
\]
which describes the evolution of the density of a solution to SGD with noise model $\Sigma = \sigma f\,I$.
Furthermore, there exists $\nu>0$ such that
\begin{align*}
\int_{\R^m}& \left( \rho(t,\theta) - \frac{f^{-1-\frac1{\eta\sigma}} (\theta)}{\int_{\R^m}f^{-1-\frac1{\eta\sigma}}(\theta')\,\d\theta'}\right)^2\,f^{1+\frac1{\eta\sigma}}(\theta)\,\d\theta \\
	&\hspace{2cm}\leq e^{-\nu t} \int_{\R^m} \left( \rho(0,\theta) - \frac{f^{-1-\frac1{\eta\sigma}} (\theta)}{\int_{\R^m}f^{-1-\frac1{\eta\sigma}}(\theta')\,\d\theta'}\right)^2\,f^{1+\frac1{\eta\sigma}}(\theta)\,\d\theta.
\end{align*}
In particular $\rho(t,\cdot)$ converges to the invariant distribution $\rho_\infty = c\,f^{-1-\frac1{\eta\sigma}}$.
\end{theorem}

We start by proving an inequality of Poincar\'e-Wirtinger type with a Hardy inequality-like weighting. The result extends Poincar\'e-Hardy inequalities in \cite{dolbeault2012improved,bonforte2010sharp} to the case where a weight-function $f$ grows slightly faster than quadratically at infinity, and may vanish (slightly slower than quadratically) at a finite collection of points. 

Specifically, we show the following.

\begin{theorem}\label{theorem poincare-hardy}
Let $f:\R^m\to [0,\infty)$ be a $C^1$-function such that there exists a finite set $\Theta =\{\bar\theta_1,\dots, \bar\theta_n\}$ and constants $0<c_1\leq C_1$ such that 
\[
c_1\leq\frac{f(\theta)}{\dist^2(\theta, \Theta)\,\big[|\log|+1\big]^2 (\theta, \Theta)}\leq C_1.
\]
Assume that $-\frac1{\eta\sigma}-1 = -\frac{m}2$. Then the measure $\mu_{f,\eta\sigma}$ with density $f^{-1-\frac1{\eta\sigma}}$ is finite and there exists $\Lambda>0$ such that
\[\showlabel \label{eq poincare-hardy new}
\int_{\R^m} \big|u-\langle u\rangle_{f,\eta\sigma} \big|^2\,f^{-\frac1{\eta\sigma}-1}\,\d\theta \leq \int_{\R^m} \big|\nabla u\big|^2\,f^{-\frac1{\eta\sigma}}\,\d\theta
\]
for all $u\in C_c^\infty(\R^m)$ where 
\[
\langle u \rangle_{f, \eta\sigma} = \frac{\int_{\R^m} u\,f^{-\frac1{\eta\sigma}-1}\,\d\theta}{\int_{\R^m} f^{-\frac1{\eta\sigma}-1}\,\d\theta}.
\]
\end{theorem}

We note that both integrals are finite for $u\in C_c^\infty(\R^m)$ because $f^{-\frac1{\eta\sigma}-1}$ is integrable. 
 Before we come to the proof, we introduce a more classical Hardy-type inequality with slightly non-standard weights -- see also \cite{masmoudi2011hardy}. 

\begin{lemma}
Let $u\in C_c^\infty(\R^m)$ and $\beta<-1$ such that the support of $u$ is contained either in $B_{1}(0)$ or $\R^m\setminus B_1(0)$.
Then
\[\showlabel\label{eq standard hardy}
\int_{\R^m} \frac{u^2(x)}{|x|^m}\, |\log|^\beta(|x|)\, \dx\leq \frac{4}{(1+\beta)^2} \int_{\R^m}\frac{u^2(x)}{|x|^{m-2}}|\log|^{2+\beta}(|x|)\,\dx\]
\end{lemma}

\begin{proof}
{\bf Step 1.} We first investigate the one-dimensional situation, where we can use a simple integration by parts argument. Assume that $u\in C^\infty(0,\infty)$ such that $\lim_{r\to 0} u^2(r) |\log|^{1+\beta}(r)  = 0$ and $u$ is supported either in $[0,1)$ or $(1,\infty)$. Then
\begin{align*}
\int_0^\infty \frac{u^2(r)}r&\,|\log|^\beta(r)\,\dr = \pm \frac1{1+\beta} \int_0^\infty u^2(r) \frac{d}{dr} |\log|^{1+\beta}(r)\dr\\
	&= -\frac2{1+\beta} \int_0^\infty uu'\,|\log|^{1+\beta}(r)\dr\\
	&=  -\frac2{1+\beta} \int_0^\infty \left(u\,r^{-1/2} |\log|^{\beta/2}(r) \right)\left(u' \,r^{1/2} |\log|^{1+\beta/2}(r)\right)\dr\\
	&\leq \frac2{|1+\beta|}\left( \int_0^\infty\frac{u^2}r\,|\log|^\beta(r)\dr\right)^\frac12\left(\int_0^\infty r\,(u')^2(r)\, |\log|^{2+\beta}(r)\dr\right)^\frac12.
\end{align*}
By rearranging terms and squaring the inequality, we find that
\[
\int_0^\infty \frac{u^2(r)}r\,|\log|^\beta(r)\dr \leq \frac4{(1+\beta)^2}\int_0^\infty r\,(u')^2(r)\,|\log|^{2+\beta}(r)\dr
\]

{\bf Step 2.} The high-dimensional case now readily reduces to the one-dimensional consideration. Since $\beta<-1$, we find that
\[
\lim_{|x|\to 0} \big|\log|x|\big|^{1+\beta}(|x|) u^2(x) =0
\]
for all $u\in C_c^\infty(\R^m)$.
Thus
\begin{align*}
\int_{\R^m} \frac{u^2(x)}{|x|^m}\,|\log|^\beta(|x|) \dx &= \avint_{S^{m-1}}\int_0^\infty \frac{u^2(r\nu)}{r^m}|\log|^\beta(r)\,r^{m-1}\dr\,\dH^{m-1}\\
	&\leq \frac4{(1+\beta)^2} \avint_{S^{m-1}}\int_0^\infty r\,\big|\nabla u\cdot \nu\big|^2(r\nu)|\log|^{2+\beta}(r)\,\dr\,\dH^{m-1}\\
	&= \frac{4}{(1+\beta)^2} \int_{\R^m}\frac{u^2(x)}{|x|^{m-2}}|\log|(|x|)^{2+\beta}\dx
\end{align*}
\end{proof}

We are now ready to prove Theorem \ref{theorem poincare-hardy}. The proof is based on a Persson-type consideration \cite{persson1960bounds}. To simplify notation, denote by $\mu_{f,\eta\sigma}$ the finite measure on $\R^m$ which has density $f^{-\frac1{\eta\sigma}-1}$ with respect to Lebesgue measure. 

\begin{proof}[Proof of Theorem \ref{theorem poincare-hardy}]
{\bf Step 1.} Assume for the sake of contradiction that there exists a sequence of functions $u_k\in C_c^\infty(\R^m)$ such that 
\[
\int_{\R^m} u_k\,f^{-\frac1{\eta\sigma}-1}\,\d\theta \equiv 0, \qquad \int_{\R^m} u_k^2\,f^{-\frac1{\eta\sigma}-1}\,\d\theta\equiv 1, \qquad \int_{\R^m} |\nabla u_k|^2\,f^{-\frac1{\eta\sigma}}\,\d\theta\to 0.
\]
Since $u_k$ is bounded in $L^2(\mu_{f,\eta\sigma})$, there exists a weakly convergent subsequence, which we also denote by $u_k$. Since $\mu_{f,\eta\sigma}$ is finite, constant functions lie in $L^2(\mu_{f,\eta\sigma})$. In particular, the weak limit $u_\infty$ satisfies
\[
\int_{\R^m} u_\infty\,f^{-\frac1{\eta\sigma}-1}\,\d\theta = 0.
\]
Since furthermore 
\[
\int_{\R^m} |\nabla u_\infty|^2\,f^{-\frac1{\eta\sigma}}\,\d\theta \leq \liminf_{k\to \infty} \int_{\R^m} |\nabla u_k|^2\,f^{-\frac1{\eta\sigma}}\,\d\theta = 0,
\]
we see that $u_\infty$ is constant. Thus $u_\infty\equiv 0$. 

{\bf Step 2.} Let $U\subseteq \R^m$ be any bounded open set such $\overline U \cap \Theta = \emptyset$. Then $f^{-\frac1{\eta\sigma}-1}, f^{-\frac1{\eta\sigma}}$ are bounded away from zero and bounded from above on $\overline U$. We thus conclude that  
\[
\sup_k \int_U u_k^2 + |\nabla u_k|^2 \d\theta <\infty,\qquad u_k\wto 0 \text{ in }L^2(U).
\]
Since $H^1(U)$ embeds into $L^2(U)$ compactly, it follows that $u_k\to u$ strongly in $L^2(U)$ and thus also
\[
\lim_{k\to\infty} \int_{U}f^{-\frac1{\eta\sigma}-1} u_k^2\,\d\theta \to 0.
\]
In particular, consider a cut-off function $\chi\in C_c^\infty(\R^m)$ such that 
\begin{enumerate}
\item $0\leq \chi\leq 1$.
\item $\chi\equiv 1$ on $U$.
\item The support of $\chi$ is compact and does not intersect $\Theta$. 
\end{enumerate}
Then
\[
\int_{\R^m} (\chi u)^2\,f^{-\frac1{\eta\sigma}-1}\,\d\theta \leq \int_{\spt(\chi)} u^2\,f^{-\frac1{\eta\sigma}-1}\,\d\theta\to 0
\]
by the first step and
\[
\int_{\R^m} \big|\nabla(\chi u)\big|^2\,f^{-\frac1{\eta\sigma}}\,\d\theta \leq 2\int_{\spt(\chi)} \big[|\nabla u|^2 + u^2|\nabla\chi|^2\,f^{-\frac1{\eta\sigma}}\big]\,\d\theta\to 0.
\]
In particular, we have constructed a sequence $\tilde u_k = u_k (1-\chi)$  such that $\tilde u_k\equiv 0$ on $U$ and
\[
\int_{\R^m} \tilde u_k^2\,f^{-\frac1{\eta\sigma}-1}\,\d\theta\to 1, \qquad \int_{\R^m} |\nabla \tilde u_k|^2\,f^{-\frac1{\eta\sigma}}\,\d\theta\to 0.
\]
In other words, we have replaced the average integral condition by a condition that the support of $u_k$ does not intersect $\overline U$ for a bounded open set $U$ of our choosing. We rename our sequences and denote $u_k:= \tilde u_k$ in the following.

{\bf Step 3.}
Let $0<\eps<1/2$ such that the balls $B_\eps(\bar\theta_i)$ are disjoint for all $\theta_i\in \Theta$ and $R>2$ such that
\[
\bigcup_{i=1}^n \overline{B_\eps(\bar\theta_i)} \subseteq B_{R/2}(0).
\]
We set
\[
U = B_R(0) \setminus \bigcup_{l=1}^K \overline{U_\eps}
\]
and decompose $u_k =\sum_{i=0}^N u_{k,i}$ where $u_{k,i}$ is supported in $B_\eps(\bar\theta_i)$ for $i=1,\dots, n$ and $u_{k,0}$ is supported on $\R^m\setminus B_R(0)$. 

{\bf Step 3.1: Estimating $u_{k,0}$.} We now observe that $\dist(\theta, \Theta) \leq |\theta|\leq 2\dist(\theta,\Theta)$ if $\theta\in \R^m\setminus B_R(0)$, so 
\begin{align*}
\int_{\R^m} u_{k,0}^2 \,f^{-1-\frac1{\eta\sigma}}\,\d\theta &\leq \left(\frac{2}{c_1}\right)^{1+\frac1{\eta\sigma}} \int_{\R^m} u_{k,0}^2 \,|\theta|^{-\left(\frac1{\eta\sigma}+1\right)2}\,|\log|^{-1-\frac1{\eta\sigma}}(|\theta|) \d\theta\\
	&= \left(\frac{2}{c_1}\right)^\frac m2 \int_{\R^m} \frac{u_{k,0}^2}{|\theta|^m} \, |\log(|\theta|)^{-m}\,\d\theta\\
	&\leq \frac{4}{(1+\beta)^2}\left(\frac{2}{c_1}\right)^\frac m2 \int_{\R^m} \frac{|\nabla u_{k,0}|^2}{|\theta|^{m-2}}\,|\log(|\theta|)^{2-m}\,\d\theta\\
	&\leq \frac{4}{(1+\beta)^2}\left(\frac{2}{c_1}\right)^\frac m2\,C_1^{\frac m2-1} \int_{\R^m}|\nabla u_{k,0}|^2 \,f^{-\frac1{\eta\sigma}}\,\d\theta
\end{align*}
where we used the weighted Hardy inequality \eqref{eq standard hardy}.

{\bf Step 3.2: Estimating $u_{k,i}$ for $i\geq 1$.} Since $\dist(\theta, \Theta) = |\theta-\bar\theta_i|$ on $B_\eps(\bar\theta_i)$, which includes the support of $u_{k,i}$, we can argue as in Step 3.1 to see that
\[
\int_{\R^m} u_{k,i}^2 \,f^{-1-\frac1{\eta\sigma}}\,\d\theta \leq \frac{4}{(1+\beta)^2}\left(\frac{2}{c_1}\right)^\frac m2\,C_1^{\frac m2-1} \int_{\R^m}|\nabla u_{k,i}|^2 \,f^{-\frac1{\eta\sigma}}\,\d\theta
\]

{\bf Step 4.} We have shown that
\[
\int_{\R^m} u_{k,i}^2 \,f^{-1-\frac1{\eta\sigma}}\,\d\theta \leq \frac{4}{(1+\beta)^2}\left(\frac{2}{c_1}\right)^\frac m2\,C_1^{\frac m2-1} \int_{\R^m}|\nabla u_{k,i}|^2 \,f^{-\frac1{\eta\sigma}}\,\d\theta
\]
In particular, this means that
\[
\int_{\R^m} u_k^2\,f^{-1-\frac1{\eta\sigma}}\,\d\theta \to 0
\]
as $k\to \infty$, contradicting the construction of the sequence. With this contradiction, the theorem is proven.
\end{proof}

We can immediately extend the Poincar\'e-Hardy inequality \eqref{eq poincare-hardy new} to the closure of $C_c^\infty(\R^m)$ with respect to the norm
\[
\|u\|_{H^1_{f,\eta\sigma}}^2 = \int_{\R^m} |\nabla u|^2\,f^{-\frac1{\eta\sigma}}\,\d\theta + \int_{\R^m} u^2\,f^{-\frac1{\eta\sigma}-1}\,\d\theta.
\]
As before, we denote the resulting Hilbert space by $H^1_{f,\eta\sigma}$, which is a subspace of $L^2_{f,\eta\sigma} = L^2(\mu_{f,\eta\sigma})$. 
 
\begin{lemma}
Assume that $f$ is as in Theorem \ref{theorem poincare-hardy} and $m\geq 4$.
Then constant functions are elements of $H^1_{f,\eta\sigma}$.
\end{lemma}

\begin{proof}
It suffices to show that $u\equiv 1\in H^1_{f,\eta\sigma}$. Let $\chi_k\in C_c^\infty(\R^m)$ be a sequence of functions such that
\[
0\leq \chi_k\leq 1, \qquad \chi_k\equiv \begin{cases}1 &|\theta|\leq 2^k\\ 0 &|\theta|\geq 2^{k+1}\end{cases}, \qquad |\nabla \chi_k|\leq 2^{-(k-1)}.
\]
Then $\chi_k\to 1$ in $L^2_{f,\eta\sigma}$ by the dominated convergence theorem. It remains to show that 
\[
\int_{\R^m} |\nabla\chi_k|^2\,f^{-\frac1{\eta\sigma}}\,\d\theta \to 0
\]
as $n\to \infty$, or at least for a subsequence. This follows from the fact that 
\begin{align*}
\int |\nabla \chi_k|^2\,f^{-\frac1{\eta\sigma}}\,\d\theta &\leq C \int_{\{2^k\leq |\theta|\leq 2^{k+1}\}} \frac{1}{|\theta|^2} \left(\frac1{|\theta|^2 \,\log^2(|\theta|)}\right)^{\frac m2-1}\d\theta\\
	&\leq C k^{2-m} \int_{2^k}^{2^{k+1}} r^{-m}\,r^{m-1}\dr\\
	&= C \,k^{3-m}.
\end{align*}
\end{proof}

In the following lemma, we prove an integration by parts type identity to make the Poincar\'e-Hardy inequality \eqref{eq poincare-hardy new} useful in the study of our evolution equations.

\begin{lemma}\label{lemma integration by parts}
Let $f$ be as in Theorem \ref{theorem poincare-hardy} and $m\geq 5$.
Let furthermore
\[
\dom(A) = \big\{u\in H^2_{loc}(\R^m\setminus\Theta) \cap L^2_{f,\eta\sigma} : \eta\sigma f\Delta u - \nabla f\cdot \nabla u \in L^2_{f,\eta\sigma}\big\}
\]
and
\[
A : \dom(A) \to L^2_{f,\eta\sigma}, \qquad Au = \eta\sigma f\Delta u - \nabla f\cdot \nabla u.
\]
Then 
\[\showlabel\label{eq integration by parts new}
\langle Au, v\rangle_{L^2_{f,\eta\sigma}} = \langle u, Av\rangle_{L^2_{f,\eta\sigma}} = - \int_{\R^m} \nabla u\cdot \nabla v\,f^{-\frac1{\eta\sigma}}\,\d\theta
\]
for all $u, v\in \dom(A)$.
\end{lemma}

\begin{proof}
{\bf Step 1.} Let $\chi_k$ be a sequence of cut-off functions such that 
\begin{enumerate}
\item $0\leq \chi_k\leq 1$,
\item $\chi_k \equiv 1$ on the set 
\[
U_k = \{ \dist(\theta, N) \geq 2^{-k}\} \cap \{|\theta| \leq 2^k\}
\]
\item $\chi_k\equiv 0$ on the set
\[
\{\dist(\theta, N)\leq 2^{-(k+1)}\} \cup \{|\theta| \geq 2^{k+1}\}
\]
\item $|\nabla \chi_k| \leq 2^{k+1}$ if $2^{-(k+1)} \leq\dist(\theta, \Theta)\leq 2^{-k}$, and
\item $|\nabla \chi_k|\leq 2^{1-k}$ if $2^k\leq |\theta|\leq 2^{k+1}$.
\end{enumerate}
The sequence may only be defined for large indices $k$ such that all the domains above are all disjoint annular regions. 
By the dominated convergence theorem, we have
\[
\lim_{k\to \infty} \langle Au, v\chi_k\rangle_{L^2_{f,\eta\sigma}} = \langle Au, v\rangle_{L^2_{f,\eta\sigma}}.
\]
On the other hand, we find that
\begin{align*}
\langle Au, v\chi_k\rangle_{L^2_{f,\eta\sigma}} &= \int_{\R^m} f^{1+\frac1{\eta\sigma}} \div\big(f^{-\frac1{\eta\sigma}}\nabla u\big)\,\chi_kv\,f^{-\frac1{\eta\sigma}-1}\,\d\theta\\
	&= \int_{\R^m} \div\big(f^{-\frac1{\eta\sigma}}\nabla u\big)\,\chi_kv\,\d\theta\\
	&= - \int_{\R^m}f^{-\frac1{\eta\sigma}}\nabla u \cdot \big(v\,\nabla \chi_k + \chi_k\nabla v\big)\,\d\theta\\
	&= - \int_{\R^m} \nabla u\cdot \nabla v\,f^{-\frac1{\eta\sigma}}\,\chi_k\d\theta - \int_{\R^m}\nabla u \cdot \nabla \chi_k\,v\,f^{-\frac1{\eta\sigma}}\d\theta.
\end{align*}
The first term on the right hand side satisfies 
\[
\lim_{k\to \infty}\int_{\R^m} \nabla u\cdot \nabla v\,f^{-\frac1{\eta\sigma}}\,\chi_k\d\theta = \int_{\R^m} \nabla u\cdot \nabla v\,f^{-\frac1{\eta\sigma}}\,\d\theta,
\]
so it remains to show that the second term approaches zero (at least along a subsequence). 

{\bf Step 2.} First, we note that
\begin{align}\label{eq integration by parts 1}
\left|\int_{\R^m}\nabla u \cdot \nabla \chi_k\,v\,f^{-\frac1{\eta\sigma}}\d\theta\right|&\leq \left(\int_{\{\nabla\chi_k\neq 0\}}|\nabla u|^2 \,f^{-\frac1{\eta\sigma}}\d\theta\right)^\frac12 \left(\int_{\R^m}v^2 \,f^{-\frac1{\eta\sigma}}|\nabla \chi_k|^2\d\theta\right)^\frac12
\end{align}
and 
\[\showlabel\label{eq integration by parts 2}
\int_{\R^m}v^2 \,f^{-\frac1{\eta\sigma}}|\nabla \chi_k|^2\d\theta \leq C \int_{\{\nabla \chi_k\neq 0\}}v^2 \,f^{-\frac1{\eta\sigma}-1} \,\log^2(|\theta|)\d\theta \leq C\,k^2\int_{\{\nabla \chi_k\neq 0\}}v^2 \,f^{-\frac1{\eta\sigma}-1} \,\d\theta.
\]
Now, consider the sets of indices
\[
I_1 = \left\{k\in \N : \int_{\{\nabla \chi_k\neq 0\}}v^2 \,f^{-\frac1{\eta\sigma}-1} \,\d\theta \leq \frac1{k\,\log k}\right\}, \qquad I_2 = \left\{k\in \N : \int_{\{\nabla\chi_k\neq 0\}}|\nabla u|^2 \,f^{-\frac1{\eta\sigma}}\d\theta \leq \frac1{k\,\log k}\right\}.
\]
Assume for the sake of contradiction that $I_1\cap I_2$ is finite. Then for all but finitely many indices we have
\[
\int_{\{\nabla \chi_k\neq 0\}}v^2 \,f^{-\frac1{\eta\sigma}-1} \,\d\theta + \int_{\{\nabla\chi_k\neq 0\}}|\nabla u|^2 \,f^{-\frac1{\eta\sigma}}\d\theta \geq \frac1 {k\,\log k}
\]
and as the domains $\{\nabla \chi_k\neq 0\}$ are disjoint for different values of $k$, we arrive at the estimate
\begin{align*}
\sum_{k=2}^\infty \frac1{k\,\log k} - C &\leq  \sum_{k=2}^\infty\left(\int_{\{\nabla \chi_k\neq 0\}}v^2 \,f^{-\frac1{\eta\sigma}-1} \,\d\theta + \int_{\{\nabla\chi_k\neq 0\}}|\nabla u|^2 \,f^{-\frac1{\eta\sigma}}\d\theta\right)\\
	&\leq \int_{\R^m}v^2 \,f^{-\frac1{\eta\sigma}-1} \,\d\theta + \int_{\R^m}|\nabla u|^2 \,f^{-\frac1{\eta\sigma}}\d\theta < \infty.
\end{align*}
The constant $C$ compensates for the absence of a finite set of terms in the infinite sum. As the series over $1/(k\,\log k)$ diverges, we have reached a contradiction. Thus, there exists a sequence of integers $k_i$ such that
\[
\int_{\{\nabla \chi_{k_i}\neq 0\}}v^2 \,f^{-\frac1{\eta\sigma}-1} \,\d\theta \leq \frac1{k_i\,\log(k_i)}, \qquad \int_{\{\nabla\chi_{k_i}\neq 0\}}|\nabla u|^2 \,f^{-\frac1{\eta\sigma}}\d\theta\leq \frac1{k_i\,\log(k_i)}
\]
for all $i\in\N$. Combining \eqref{eq integration by parts 1} and \eqref{eq integration by parts 2} along the sequence $k_i$, we find that
\[
\left|\int_{\R^m}\nabla u \cdot \nabla \chi_{k_i}\,v\,f^{-\frac1{\eta\sigma}}\d\theta\right|\leq \left(\frac1{k_i\,\log(k_i)}\right)^\frac12\left(\frac1{k_i\,\log(k_i)}\,k_i^2\right)^\frac12 \leq \frac1{\log k_i}\to 0.
\]
\end{proof}

\begin{proof}[Proof of Theorem \ref{theorem convergence appendix}]
The proof now proceeds as in the non-singular case, as the estimates of the operators are the same. Since we only require $u$ to be in $H^2_{loc}(\R^m\setminus \Theta)$, the local regularity result in \cite[Theorem 8.8]{gilbarg2015elliptic} for the elliptic problem still suffices. See the proofs of Corollary \ref{corollary mild solution} and \ref{theorem convergence} for details.
\end{proof}

\begin{remark}
A Poincar\'e-inequality like \eqref{eq poincare-hardy new} can be proved also under different assumptions with virtually the same proof and a slightly different weighted Hardy-inequality, e.g.\ if 
\begin{itemize}
\item there exists a compact $n$-dimensional $C^2$-manifold $N$ with $n<m$ such that $f(\theta) = 0$ if and only if $\theta\in N$,
\item there exist constants $c_1, C_1, \eps >0$ and $\gamma_1\in(0,2]$ such that 
\[
c_1\,\dist(\theta, N)^{\gamma_1} \leq f(\theta) \leq C_1\,\dist(\theta, N)^{\gamma_1}\qquad \forall\ \theta \text{ s.t. } \dist(\theta, N) <\eps.
\]
\item there exist $c_2, C_2, R>0$, $\gamma_2\geq 2$ such that 
\[\showlabel\label{eq growth minimum}
c_2\,|\theta|^{\gamma_2} \leq f(\theta) \leq C_2\,|\theta|^{\gamma_2} \qquad\forall\ |\theta|\geq R.
\]
\item the parameters satisfy the compatibility conditions $\gamma_1 < \frac{m-n}{n} \gamma_2$ and $\eta,\sigma>0$ such that 
\[\showlabel\label{eq finite measure condition appendix}
\frac{m}{\gamma_2} < 1+\frac1{\eta\sigma} < \frac{m-n}{\gamma_1}.
\]
\end{itemize}
The difficulty is the proof of the integration by parts identity of Lemma \ref{lemma integration by parts}, which converts the Poincar\'e inequality into useable spectral information. This proof hinges crucially on the fact that $f$ behaves (almost) quadratically both at its set of minimizers and at infinity and does not carry over to the case where $\gamma_1, \gamma_2\neq 2$. 
\end{remark}

\end{document}